\documentclass{article}

\def\conference{arxiv}  
\def\neuripsconf{neurips}
\def\icmlconf{icml}
\def\arxivconf{arxiv}

\usepackage[table,xcdraw]{xcolor}

\ifx\conference\neuripsconf
    \usepackage[nonatbib]{styles/neurips_2025}
    \usepackage[numbers,sort&compress]{natbib}
\else\ifx\conference\arxivconf
    \usepackage{styles/arxiv}
\else
    \usepackage{styles/icml2025}
\fi

\usepackage{enumitem}

\usepackage[utf8]{inputenc} 
\usepackage[T1]{fontenc}    
\usepackage{amsmath}
\usepackage{hyperref}       
\usepackage{url}            
\usepackage{booktabs}       
\usepackage{amsfonts}       
\usepackage{nicefrac}       
\usepackage{microtype}      
\usepackage{xspace}
\usepackage{cleveref}
\usepackage[breakable]{tcolorbox}
\usepackage{listings}
\usepackage{bbold}
\usepackage{enumitem}
\usepackage{algorithm}
\usepackage{algpseudocode}
\usepackage{ifthen}
\usepackage{fontawesome5}

\usepackage{amsthm}
\usepackage{math_commands}
\usepackage{amssymb}
\usepackage{mathtools}

\usepackage{pifont}
\newcommand{\cmark}{\ding{51}}%
\newcommand{\xmark}{\ding{55}}%

\newcommand{\method}{\textsc{Self-Study}\xspace}

\newcommand{\artifact}{\textsc{Cartridge}\xspace}
\newcommand{\artifacts}{\textsc{Cartridges}\xspace}

\newcommand{\longhealth}{\textsc{LongHealth}\xspace}

\newcommand{\mtob}{\textsc{MTOB}\xspace}
\newcommand{\qasper}{\textsc{QASPER}\xspace}

\newcommand{\genconvo}{\textsc{GenConvo}\xspace}

\newcommand{\llamathree}{\textsc{Llama 3}\xspace}
\newcommand{\llamathreeb}{\textsc{Llama-3B}\xspace}
\newcommand{\llamaeightb}{\textsc{Llama-8B}\xspace}

\newcommand{\eg}{\textit{e.g.}\xspace}
\newcommand{\ie}{\textit{i.e.}\xspace}

\newcommand{\etal}{\textit{et al.}\xspace}

\newcommand{\ctx}{\mathcal{C}}
\newcommand{\subctx}{\tilde{\mathbf{c}}}
\newcommand{\seed}{\mathbf{s}}
\newcommand{\ctxrep}{Z}

\newcommand{\queries}{Q}
\newcommand{\query}{q}
\newcommand{\resps}{R}
\newcommand{\resp}{r}
\newcommand{\vocab}{\mathcal{V}}
\newcommand{\llm}{\mathcal{F}}

\newcommand{\numtrain}{m_\text{train}}

\newtheorem{theorem}{Theorem}
\newtheorem{lemma}{Lemma}

\newtheorem{definition}{Definition}

\newtcolorbox{examplebox}[1][]{
    colback=lightgray!10,
    colframe=black,
    boxrule=0.75pt,
    title=#1,
    fonttitle=\bfseries,
    left=3pt,
    right=3pt,
    top=2pt,
    bottom=2pt,
    breakable,  
}

\newtcolorbox{exampleboxcode}[1][]{
    colback=lightgray!10,
    colframe=blue,
    boxrule=0.75pt,
    title=#1,
    fonttitle=\ttfamily,
    left=3pt,
    right=3pt,
    top=2pt,
    bottom=2pt,
}

\definecolor{codegreen}{rgb}{0,0.6,0}
\definecolor{codegray}{rgb}{0.5,0.5,0.5}
\definecolor{codepurple}{rgb}{0.58,0,0.82}
\definecolor{backcolour}{rgb}{0.97,0.97,0.97}

\lstdefinestyle{codestyle}{
    commentstyle=\color{codegreen},
    keywordstyle=\color{magenta},
    numberstyle=\tiny\color{codegray},
    stringstyle=\color{codepurple},
    basicstyle=\ttfamily\scriptsize,
    breakatwhitespace=false,         
    breaklines=true,                 
    captionpos=b,                    
    keepspaces=true,                 
    numbers=left,                    
    numbersep=5pt,                  
    showspaces=false,                
    showstringspaces=false,
    showtabs=false,                  
    tabsize=2,
    xleftmargin=10pt,                
    xrightmargin=10pt                
}
\newcommand{\repetitiveMQAR}{\text{repetitive MQAR}}

\newcommand{\contextSize}{N}
\newcommand{\numPairs}{m}
\newcommand{\numQueries}{n}
\newcommand{\timestep}{t}

\newcommand{\modelDim}{d}

\newcommand{\state}{\mW}

\newcommand{\cacheMatrix}{\mW}
\newcommand{\cacheMatrixTime}[1]{\cacheMatrix^{(#1)}}

\newcommand{\kvSet}{S}

\newcommand{\kvpair}[1]{(\key{#1}, \val{#1})}

\newcommand{\key}[1]{\vk^{(#1)}}
\newcommand{\keyT}[1]{\left(\vk^{(#1)}\right)^\top}

\newcommand{\val}[1]{\vv^{(#1)}}

\newcommand{\error}[2]{\eps_{#1,#2}}

\newcommand{\errs}[3]{C_{#1,#2}^{(#3)}}

\makeatletter
\newcommand{\github}[1]{%
   \href{#1}{\faGithubSquare}%
}
\makeatother

\ifx\conference\neuripsconf

\title{\artifacts: Lightweight and general-purpose long context representations via self-study}
\author{%
  David S.~Hippocampus\thanks{Use footnote for providing further information
    about author (webpage, alternative address)---\emph{not} for acknowledging
    funding agencies.} \\
  Department of Computer Science\\
  Cranberry-Lemon University\\
  Pittsburgh, PA 15213 \\
  \texttt{hippo@cs.cranberry-lemon.edu} \\
}
\begin{document}

\maketitle

\else\ifx\conference\icmlconf
\icmltitlerunning{Submission and Formatting Instructions for ICML 2025}

\begin{document}

\twocolumn[
\icmltitle{\artifacts: Lightweight and general-purpose long context representations via self-study}


\icmlsetsymbol{equal}{*}

\begin{icmlauthorlist}
\icmlauthor{Firstname1 Lastname1}{equal,yyy}
\icmlauthor{Firstname2 Lastname2}{equal,yyy,comp}
\icmlauthor{Firstname3 Lastname3}{comp}
\icmlauthor{Firstname4 Lastname4}{sch}
\icmlauthor{Firstname5 Lastname5}{yyy}
\icmlauthor{Firstname6 Lastname6}{sch,yyy,comp}
\icmlauthor{Firstname7 Lastname7}{comp}
\icmlauthor{Firstname8 Lastname8}{sch}
\icmlauthor{Firstname8 Lastname8}{yyy,comp}
\end{icmlauthorlist}

\icmlaffiliation{yyy}{Department of XXX, University of YYY, Location, Country}
\icmlaffiliation{comp}{Company Name, Location, Country}
\icmlaffiliation{sch}{School of ZZZ, Institute of WWW, Location, Country}

\icmlcorrespondingauthor{Firstname1 Lastname1}{first1.last1@xxx.edu}
\icmlcorrespondingauthor{Firstname2 Lastname2}{first2.last2@www.uk}

\icmlkeywords{Machine Learning, ICML}

\else\ifx\conference\arxivconf

\title{Cartridges: Lightweight and general-purpose long context \\ representations via self-study}

\author{\textbf{Sabri Eyuboglu} \textsuperscript{1}$^*$
\quad
\textbf{Ryan Ehrlich} \textsuperscript{1}$^*$
\quad
\textbf{Simran Arora} \textsuperscript{1,2}$^* $ 
\quad
\textbf{Neel Guha} \textsuperscript{1}
\quad
\textbf{Dylan Zinsley} \textsuperscript{3}
\quad
\textbf{Emily Liu} \textsuperscript{1} \\
\textbf{Will Tennien}  \textsuperscript{1}
\quad
\textbf{Atri Rudra} \textsuperscript{3}
\quad
\textbf{James Zou} \textsuperscript{1}
\quad
\textbf{Azalia Mirhoseini} \textsuperscript{1}
\quad
\textbf{Christopher Ré} \textsuperscript{1} \\[5pt]
\textsuperscript{1}Stanford University \quad \textsuperscript{2} Caltech \quad \textsuperscript{3}University at Buffalo \quad \quad * Equal contribution \\[3pt]
\faEnvelope \enspace \texttt{eyuboglu@stanford.edu, rehrlich@stanford.edu, simarora@stanford.edu} \\
\faGithubSquare \enspace \texttt{\href{https://github.com/HazyResearch/cartridges}{HazyResearch/cartridges}}
}

\begin{document}

\maketitle

\fi

\begin{abstract}
Large language models are often used to answer queries grounded in large text corpora (\eg codebases, legal documents, or chat histories) by placing the entire corpus in the context window and leveraging in-context learning (ICL). 
Although current models support contexts of 100K–1M tokens, this setup is costly to serve because the memory consumption of the KV cache scales with input length. 
We explore an alternative: training a smaller KV cache offline on each corpus. 
At inference time, we load this trained KV-cache, which we call a \artifact, and decode a response. Critically, the cost of training a \artifact can be amortized across all the queries referencing the same corpus. 
However, we find that the naive approach of training the \artifact with next-token prediction on the corpus is not competitive with ICL. 
Instead, we propose \method, a training recipe in which we generate synthetic conversations about the corpus and train the \artifact with a context-distillation objective. 
We find that \artifacts trained with \method replicate the functionality of ICL, while being significantly cheaper to serve. 
On challenging long-context benchmarks, \artifacts trained with \method match ICL performance while using $38.6\times$ less memory and enabling $26.4\times$ higher throughput. \method also extends the model’s effective context length (\eg from 128k to 484k tokens on MTOB) and surprisingly, leads to \artifacts that can be composed at inference time without retraining.
\end{abstract}

\vspace{-2mm}
\section{Introduction}
Large language model (LLM) users often place large text corpora into the context window. For instance, a user or organization may use LLMs to understand codebases~\cite{nam2024using}, financial documents~\cite{islam2023financebench}, legal texts~\cite{guha2023legalbench, zheng2025reasoning}, textbooks~\cite{ouellette2025can}, or personal files~\cite{arora2022can}. 
LLMs excel here due to in-context learning (ICL), enabling accurate responses to diverse queries (e.g., factual Q\&A, summarization, code generation)~\cite{dong2022survey}.

Despite its flexibility, this usage paradigm is costly to serve. ICL requires maintaining a KV cache that grows linearly with the input length. For example, LLaMA 70B needs 84 GB of memory (at 16-bit precision) to answer a single question over a 128k-token context~\cite{dubey2024llama3}. This severely limits user throughput: on a single H100 GPU, LLaMA 8B’s peak throughput (tokens/s) drops by $77\times$ when increasing the context from 1k to 120k tokens (\Cref{fig:micros}). 

Prior work has thus explored ways to reduce KV cache memory usage. For instance, prompt compression methods reduce the number of tokens stored in the cache using summarization, or self-information filtering~\cite{jiang2023llmlingua,li2023unlocking,chuang2024learning}, while KV cache compression techniques directly compress the stored key-value pairs~\cite{ge2023model,zhang2023h2o,tang2024quest,oren2024transformers}. Unfortunately, there are memory-quality tradeoffs associated with these methods: in experiments on challenging long-context tasks, we find that performance  degrades rapidly when applying these methods with compression ratios greater than $2\times$ (see \Cref{fig:tradeoff-within}).

Motivated by the observation that the cost of preparing a KV cache can be amortized across many queries that reference the same corpus, we explore a complementary approach based on offline training. Given a specific text corpus (\eg a patient's medical record) we freeze the LLM and train a smaller KV cache offline by backpropagating loss into the key and value vectors in a process essentially equivalent to prefix tuning~\cite{li2021prefix,lester2021power}. We call the trained KV cache representing the corpus  a ``\artifact.'' At inference time, we load the trained \artifact, append the user’s messages, and decode. Because users repeatedly reference the same corpora  (\eg SEC filings, codebase, personal files), each \artifact can be trained once offline and reused. This approach also integrates cleanly with existing inference servers, which are already designed to manage per-user KV caches~\cite{kwon2023efficient,zheng2024sglang,juravsky2025tokasaurus,ye2025flashinfer}.
 
\begin{figure*}[t]  
    \centering
    \includegraphics[width=\textwidth]{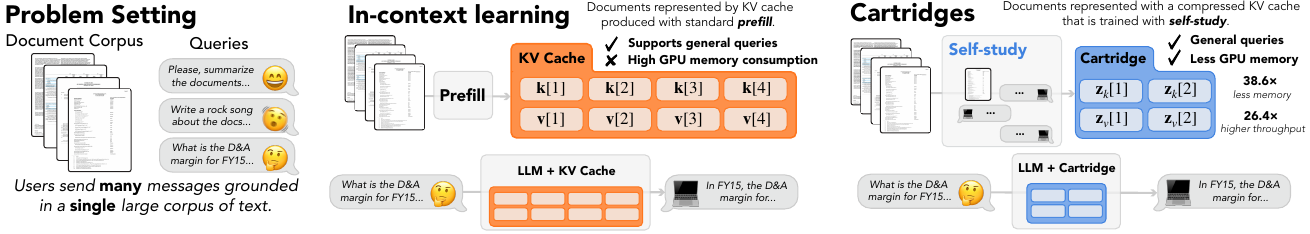}

    \caption{
        \textbf{Producing \artifacts via self-study}. For a given document corpus, we train a \artifact by distilling the corpus into a parameterized KV cache through a process we call \method. At inference time, this \artifact can be loaded into an LLM, which can then be used to answer diverse queries about the corpus, simulating in-context analysis of the corpus while requiring substantially less memory. 
    }
    \label{fig:banner}
\vspace{-3mm}
\end{figure*}

Achieving ICL-equivalent functionality requires \artifacts to satisfy two non-trivial desiderata. First, \artifacts should replicate the generality of ICL, and provide accurate responses across diverse user prompts~\cite{dong2022survey}. 
Second, \artifacts should replicate ICL's  structural awareness---its ability to reason over document structure, and understand how distant parts of a corpus relate or depend on each other (an ability that degrades when using lossy KV-cache compression methods).
%
It is unclear if there is a procedure that satisfies these desiderata, while providing memory efficiency.

The natural baseline approach is to train a \artifact with a next-token prediction objective on the raw corpus. Excitingly, this yields \artifacts that memorize the corpus \textit{perfectly} using $107\times$ less memory than the KV-cache. However, the resulting \artifacts are not general - they degrade the LM's ability to respond to diverse questions beyond regurgitating the corpus (\Cref{fig:micros}). 

To address these challenges and produce general, structurally aware \artifacts for any text corpus, we propose an automated method called \method. \method has two steps:
\begin{enumerate}[leftmargin=*]
    \item \textbf{Synthetic data generation} (\Cref{sec:method-data}): We generate synthetic training data by prompting the model to quiz itself about the corpus content, resulting in a synthetic conversation trace. Training on these lets us avoid training on the same exact text multiple times and improves generality (see \Cref{fig:micros}).
To support corpora that exceed the effective context length of the model, we chunk the corpus when generating synthetic conversations. 
We also curate a set of seed prompts that bias the synthetic conversations towards global reasoning and improve structural awareness (see \Cref{fig:ablations} right). 
    \item \textbf{Context distillation} (\Cref{sec:method-objective}): We train on the synthetic conversations using a context-distillation objective~\cite{bhargava2024prompt,snell2022learning}, which aligns the \artifact-augmented model's next-token distributions with the distributions of the model with the corpus in context. We find that the context distillation substantially improves the quality of the \artifacts compared to next-token-prediction (see \Cref{fig:ablations} center).
\end{enumerate}

%


In summary, given a large corpus of text, our goal is to train a small virtual KV cache, termed \artifact, that when used by the model, mimics the conversational behavior of the model with the entire corpus in context. To do this, we generate synthetic conversations and train the \artifact on them with a context distillation objective --- a recipe we call \method.

\textbf{Evaluations.} We evaluate \artifacts trained with {\method} on a set of challenging benchmarks that pair a single large text corpus ($100$k-$484$k tokens) with a diverse set of queries~\cite{islam2023financebench,adams2024longhealth,tanzer2023benchmark}. 
We make three claims. \textbf{First}, \artifacts extends the quality-memory frontier---averaged across the benchmarks, \artifacts produced with \method match ICL quality while consuming $38.6\times$ less memory, enabling a $26.4\times$ increase in peak throughput (tokens per second) when serving many users with different corpora. These memory reductions and speedups represent an order of magnitude improvement over state-of-the-art cache compression baselines (\eg DuoAttention~\cite{xiao2024duoattention}). \textbf{Second}, \artifacts enables context length extrapolation. On the MTOB benchmark~\cite{tanzer2023benchmark}, where models must translate from Kalamang, a low-resource language, into English, we use \method with \llamaeightb to construct a small \artifact from a $484$k token textbook.
This \artifact outperforms ICL over the first $130,000$ tokens of the textbook by $11.0$ chrF points and matches the ICL performance over a curated subset of the textbook.
\textbf{Third}, \method also yields \artifacts that are composable without joint optimization: multiple \artifacts can be concatenated and queried together, emulating ICL's ability to flexibly answer queries over multiple documents concatenated in context (see \Cref{fig:composition}).

Additionally, we carefully ablate the design decisions in \method and \artifacts (\Cref{sec:results-ablations} and \Cref{app:results}). Notably, we compare \artifacts parameterized as a KV cache~\cite{li2021prefix} with \artifacts parameterized as a LoRA \cite{hu2022lora} and find that KV cache parameterization performs better on both in-domain and out-of-domain tasks.

In this work, we demonstrate how offline KV cache training can dramatically reduce the cost of serving language models in settings where users repeatedly include the same text corpora in context. 
We hope that these cost reductions could enable new applications that are currently intractable, like coding agents with full-repository context or long-term memory in chatbots.

\vspace{-2mm}
\section{Preliminaries}

\begin{figure}[t!]
\centering
\scalebox{0.9}{ 
\setlength{\tabcolsep}{4pt}
\begin{tabular}{@{}lcccc@{}}
\toprule
\textit{Method} &
\begin{tabular}[c]{@{}c@{}}Consumes limited \\ memory \end{tabular} &
\begin{tabular}[c]{@{}c@{}}Retains corpus  \\ information\end{tabular} &
\begin{tabular}[c]{@{}c@{}}Supports diverse \\ prompts \end{tabular} &\\
\midrule
In-context learning 
& \xmark & \cmark & \cmark  \\
Prompt / KV cache compression 
& \cmark & \xmark & \cmark \\
\artifact + Next-token-prediction
& \cmark & \cmark & \xmark  \\
\rowcolor[HTML]{EFEFEF}  \artifact + \method 
& \cmark & \cmark & \cmark\\
\bottomrule
\end{tabular}
}
\caption{\small \textbf{Comparing KV caching strategies.} \artifact improves memory efficiency, while retaining the quality of in-context learning across a broad set of prompts. \cmark~ indicates a strength and \xmark~ indicates a limitation.}
\end{figure}

We begin by discussing related work (\Cref{sec:related_work}), formalizing our problem (\Cref{sec:problem-setup}), and providing background on language models and KV caches (\Cref{sec:lm-background}). 

\label{sec:preliminaries}
\subsection{Related work}\label{sec:related_work}

\ifx\conference\icmlconf
\begin{figure*}[t!]  
    \centering
    \includegraphics[width=\textwidth]{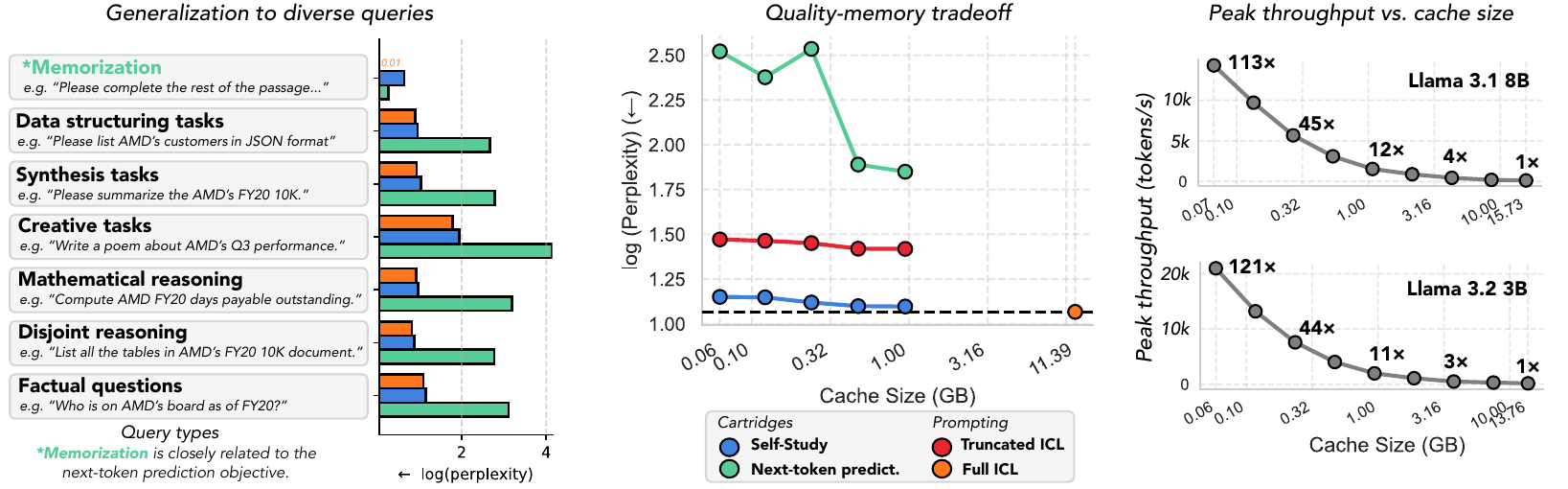}

    \caption{
        \textbf{\artifacts trained with \method balance the generality and memory consumption tradeoff.} 
        We compare four methods on the \genconvo dataset: \artifacts trained with next-token prediction over $\ctx$, \artifacts trained with \method, full ICL, and truncated ICL, a prompt compression method in which we truncate the $\ctx$ to the first $k$ tokens.   
        (\textbf{Left}) We evaluate on different slices from the \genconvo dataset. \artifacts trained with next-token prediction performs well on memorization queries, which resemble it's training distribution, but cannot generalize to other queries like the other methods.
        (\textbf{Center}) The $x$-axis measures the size of the KV cache in GB for the different methods.
        The $y$-axis shows log-perplexity on the \genconvo dataset averaged over the query types.
        (\textbf{Right}) Peak throughput (tokens/s) measured for different cache sizes for \llamathreeb and \llamaeightb with SGLang~\cite{zheng2024sglang} on an 1xH100 (See \Cref{app:results}).
    } 
    \label{fig:micros}
\end{figure*}

\fi

\textit{See Appendix \ref{app:related-work} for a detailed discussion of prior work.}

\vspace{-1mm}
\paragraph{Parameter Efficient Fine-Tuning and Knowledge Injection}
In order to adapt a language model to a specific task or domain, practitioners commonly train a small number of parameters, which augment or modify the original model~\cite{hu2022lora,li2021prefix,lester2021power, meng2024pissa,zaken2021bitfit}. 
In particular, low rank adaptation~\cite{hu2022lora}, where linear layers are adapted with low rank updates, is the de facto parameter efficient fine-tuning technique. 
In our work, we build upon a less popular technique, prefix-tuning~\cite{li2021prefix,lester2021power}, where we optimize internal activations for a set of ``virtual'' tokens preceding the input. 
 
Recent works on \textit{knowledge injection} apply LoRA (or variants~\cite{mao2025lift}) to store a text corpus in a small number of parameters~\cite{zhang2023plug,xiao2023plug,kujanpaa2024knowledge,mao2025lift,su2025parametricrag}. 
This allows models to answer queries using parameteric knowledge as opposed to ICL.
The earliest methods in this line of work inject knowledge with next-token prediction objectives on the corpus~\cite{zhang2023plug,xiao2023plug,kuratov2025cramming}.
Excitingly, recent and concurrent work has also demonstrated the power of synthetic data~\cite{mao2025lift,su2025parametricrag} and context-distillation objectives~\cite{kujanpaa2024knowledge,caccia2025training} in knowledge injection.
In contrast to our work, these papers do not focus on memory reductions or throughput improvements enabled by knowledge injection. 
Furthermore, they do not use a prefix-tuning parameterization, formulate synthetic data generation as a conversation, or seed the conversation with diverse seed prompts, which we find to be critical for performance on long-context tasks and out-of-domain generalization.

Related to our analysis of \artifact composition are a number of works that compose multiple different parameter-efficient adapters through various aggregation operations ~\cite{zhao2024merging,huang2023lorahub,xiao2024configurable,zhao2024loraretriever,yadav2024survey,wu2024mixture,gou2023mixture,li2024mixlora}.

\vspace{-1mm}
\paragraph{Prompt and KV-cache compression} Because the size of the KV cache is a major determinant of language model serving cost, many works have proposed techniques to reduce the size of the cache. One set of approaches focus on making the prompt smaller---explicit methods alter the prompt text through summarization and filtering ~\cite{jiang2023llmlingua,li2023unlocking,chuang2024learning,zhang2024adacomp,pan2024llmlingua}, while implicit methods compress prompt representations into a set of ``soft'' tokens~\cite{chevalier2023adapting,yen2024long,ge2023context,mu2023learning,qin2023dodo, lester2021power}. Another set of approaches exploits observations about the structure of the KV cache~\cite{yu2024effectively,chang2024palu,kim2024lexico}, often finding that because a small number of keys dominate the attention scores of subsequent queries, non-impactful key-value pairs (or tokens) can be dropped~\cite{ge2023model,zhang2023h2o,tang2024quest,oren2024transformers, li2024snapkv} or merged~\cite{wang2024model,zhang2024cam,wan2024d2o}. 
Compared with our work, these methods use relatively little compute to compress the KV cache. We focus on the setting where scaling the amount of compute used to compress the KV cache makes sense because contexts are shared across many requests.

\vspace{-1mm}
\paragraph{Architectural changes} A large body of work has studied architectural changes to the original multi-head attention operation~\cite{vaswani2017attention} with the aim of reducing the memory footprint of the KV cache or replacing it with a memory object of constant size. 
Unlike \method and the compression approaches discussed above, which can be readily applied to any pre-trained Transformer, these architectural changes typically require retraining the model from scratch or using complex architecture conversion techniques~\cite{zhang2024lolcats}.

In order to reduce the memory footprint of the KV cache, these architectures leverage sparsity ~\cite{beltagy2020longformer,child2019generating,zaheer2020big, team2024gemma}, reduce the number of key and value heads~\cite{shazeer2019fast, ainslie2023gqa}, make the key and value heads low-rank~\cite{liu2024deepseek}, or replace the KV cache with a constant-size memory object~\cite{ zhang2025tensor, arora2024simple,gu2023mamba,yang2024fla,yang2024gla}.
In particular, grouped-query attention~\cite{ainslie2023gqa} is the de-facto multi-head attention variant, used in frontier language models like Llama 3~\cite{dubey2024llama3}. In our experiments, we compare against ICL with grouped-query attention.
Other variants --- such as multi-head latent attention~\cite{liu2024deepseek} or linear attention~\cite{arora2024simple,gu2023mamba} --- are gaining popularity and feature in large-scale reasoning models~\cite{guo2025deepseek} and hybrid models~\cite{li2025minimax,blakeman2025nemotron,team2024jamba}.

Most related to our work are recent architectures (\eg Titans~\cite{behrouz2024titans}, TTT~\cite{sun2024learning}) that use a constant-sized memory object (like in linear attention) but apply gradient descent-like memory updates~\cite{sun2024learning,yang2025parallelizinglineartransformersdelta,behrouz2025atlas,behrouz2024titans,behrouz2025s}.
Like our work, these architectures are motivated by the observation that gradient descent is very effective at compressing text into constant space and demonstrate the promise of using gradient descent at test time for long-context tasks. 
In contrast with our work, these architectures need to be trained from scratch, they have not been validated on large scale models, and do not match the quality of attention on recall-intensive tasks~\cite{arora2024simple,behrouz2025atlas}.

\ifx\conference\arxivconf
\begin{figure*}[t!]  
    \centering
    \includegraphics[width=\textwidth]{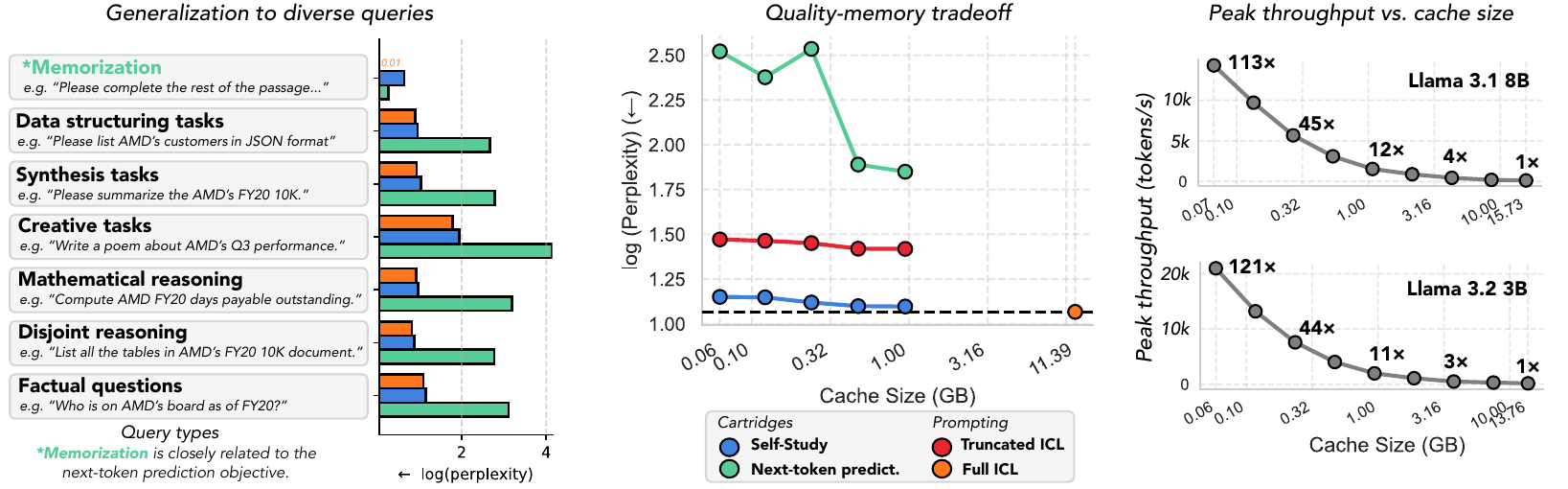}

    \caption{
        \textbf{\artifacts trained with \method balance the generality and memory consumption tradeoff.} 
        We compare four methods on the \genconvo dataset: \artifacts trained with next-token prediction over $\ctx$, \artifacts trained with \method, full ICL, and truncated ICL, a prompt compression method in which we truncate the $\ctx$ to the first $k$ tokens.   
        (\textbf{Left}) We evaluate on different slices from the \genconvo dataset. \artifacts trained with next-token prediction performs well on memorization queries, which resemble it's training distribution, but cannot generalize to other queries like the other methods.
        (\textbf{Center}) The $x$-axis measures the size of the KV cache in GB for the different methods.
        The $y$-axis shows log-perplexity on the \genconvo dataset averaged over the query types.
        (\textbf{Right}) Peak throughput (tokens/s) measured for different cache sizes for \llamathreeb and \llamaeightb with SGLang~\cite{zheng2024sglang} on an 1xH100 (See \Cref{app:results}).
    } 
    \label{fig:micros}
\end{figure*}

\fi

\vspace{-2mm}

\vspace{-1mm}
\subsection{Problem setup}\label{sec:problem-setup}
We assume a setting in which users issue a stream of diverse queries about a common corpus of text. We denote the corpus as $\ctx$ and the query set as $\queries = \{\query_1, \query_2,\hdots,\query_m\}$. Illustrative examples of $\ctx$ include legal filings, financial documents, code repositories, chat histories, and medical records. 

\begin{examplebox}[Example: Financial Analysis]
    $\ctx$ may correspond to the 2022 Form 10-K filing~\cite{sec_read10k} for AMD, which is almost 100k tokens. The queries an analyst might ask an LLM to answer with respect to this form are diverse, including: (1) recalling factual information, (2) performing mathematical reasoning over values, or (3) even generating creative responses (e.g., a poem) grounded in the 10-K's information. 
\end{examplebox}

Let $\resps = \{\resp_1, \resp_2,\hdots,\resp_m\}$ denote the responses the LLM produces for the queries. We have two objectives. First, we wish to maximize the quality of responses $\resps$ under some quality metric (\eg accuracy). Second, we wish to minimize the LLM's memory footprint while it is answering questions with respect to the document. This is because larger memory footprints decrease throughput and necessitate more hardware to serve the same number of users (\Cref{fig:micros}, Right).

\vspace{-2mm}
\subsection{Language models and KV caches} 
\label{sec:lm-background}
Recall that an LLM $\llm$ accepts as input a sequence of $N$ tokens $\mathbf{x} \in \mathcal{V}^n$  drawn from a discrete vocabulary $\vocab \subset \mathbb{Z}$ of tokens, each represented by a unique integer. 
The output, which we denote $\llm(\cdot | \mathbf{x})$, corresponds to a categorical distribution over a vocab $\vocab$ conditioned on the prefix  $\mathbf{x} \in \vocab^n$.

Inside the language model, each token $x[i]$ in $\mathbf{x}$ is embedded into a $d$-dimensional space, yielding a matrix $\mathbf{u} \in \mathbb{R}^{n\times d}$. 
The matrix $\mathbf{u}$ is passed through a stack of $L$ model layers, which each mix the matrix along the $n$ and $d$ dimensions, with layer $\ell$ outputting $\mathbf{y}^l \in \mathbb{R}^{n\times d}$. 
The final $\mathbf{y}^L$ is mapped to the logits over $\vocab$ with a linear projection.

Most modern language models use the Transformer architecture based on self-attention~\cite{vaswani2017attention}. Given an input $\mathbf{u} \in \mathbb{R}^{n\times d}$ for sequence length $n$ and embedding dimension $d$, it computes the output $\mathbf{y}^l \in \mathbb{R}^{n \times d}$ via the softmax over projections 
\(
\mathbf{q}, \mathbf{k}, \mathbf{v} = \mathbf{u} \mathbf{W}_q, \mathbf{u} \mathbf{W}_k, \mathbf{u} \mathbf{W}_v
\):
\begin{equation}
\mathbf{y}[i] = \sum_{j=1}^i\frac{\exp(\mathbf{q}[i]^\top \mathbf{k}[j] /\sqrt{d}) \mathbf{v}[j] }{\sum_{t = 1}^{i} \exp(\mathbf{q}[i]^\top \mathbf{k}[t] /\sqrt{d})} 
\label{eq:softmax_attention}
\end{equation}
where weight matrices $\mW_q$, $\mW_k$ and $\mW_v$ for each layer are learned during training. 

When generating from $\llm$, we generate one token at a time by sampling from $\llm(\cdot \mid \mathbf{x})$ and appending the sampled token to $\mathbf{x}$.
Critically, the attention operator is causal: every output $\mathbf{y}[i]$ is conditioned on prior tokens. 
This allows us to avoid recomputing the keys and values for the prior tokens by storing them in a KV cache $\{\mathbf{k}[j], \mathbf{v}[j]\}_{j=1}^{i}$, which grows in $i$.
Thus, generation proceeds in two phases: (1) \textit{prefill}, where we compute the KV cache for the initial prompt $\mathbf{x}$ and (2) \textit{decode}, where we generate the response token by token and append to the KV cache.
After prefill, if $\mathbf{x}$ consists primarily of the corpus $\ctx$, the KV cache effectively serves as a representation of the corpus $\ctx$.
This is why including a long corpus $\ctx$ in the context $\mathbf{x}$ produces large memory footprints, as the size of the KV cache scales linearly in the length of $\mathbf{x}$.

\section{The \artifact paradigm}
\label{sec:artifact}

In this section, we describe the \artifact paradigm, in which we generate representations of the corpus $\ctx$ offline with training, instead of the standard approach of constructing them on-the-fly with prefill.

\vspace{-2mm}
\subsection{Formalizing \artifacts}\label{sec:desiderata}
\label{sec:artifact-desiderata}

Our goal is to train a \artifact for a given corpus $\ctx$. A \artifact is a small set of parameters $\ctxrep \in \mathbb{R}^*$ (\ie an adapter~\cite{li2021prefix,hu2022lora}) that augments an LLM $\llm$ and causes it to behave as if it had $\ctx$ in its context window. Formally, let $\llm_Z( \cdot | \query )$ denote the distribution of $\llm$ augmented with $Z$ given a query $\query$. For all $\query \in \queries$, we want to ensure that samples $r_Z \sim \llm_Z(\cdot | \query)$ are as good or better than the ICL sample $r_q \sim \llm(\cdot | \ctx \oplus \query)$, according to some query-specific scoring function. In order for $\llm_Z(\cdot | \query)$ to match or exceed the behavior of $\llm(\cdot | \ctx \oplus \query)$, three important criteria should be met.

\begin{itemize}[leftmargin=*]
    \item \textbf{Displays generality}: Because $\queries$ might span a diverse range of question types (e.g., mathematical reasoning, factual recall comprehension, summarization, and more), it is essential that $\llm_Z$ can generalize across different $\query \in \queries$. \textbf{This is non-trivial because $\queries$ is unknown when $Z$ is being learned offline.} If $\llm_Z$ does not generalize, then practitioners may need to learn different $Z$ for different distributions of queries, which increases the cost of the \artifact. Ideally, $Z$ should only need to be learned once, yet work for multiple types of queries.
    \item \textbf{Captures long range dependencies}: $Z$ should also capture long range dependencies contained within $\ctx$. In many settings, correctly answering different $\query \in \queries$ requires reasoning about the order of information presented in $\ctx$. It is not clear how to capture these dependencies in $Z$.
    \item \textbf{Capable of composition}: Ideally, the representation of $Z$ and mechanism by which $\llm$ utilizes it could allow for composition, without any particular joint training of \artifacts. Given $Z_1$ and $Z_2$ corresponding to $\ctx_1$ and $\ctx_2$, ideally $\llm_{[Z_1, Z_2]}(\query)$ is similar to $\llm(\cdot | \ctx_1 \oplus \ctx_2 \oplus \query])$
\end{itemize}

\vspace{-2mm}
\subsection{Parameterizing \artifacts}\label{sec:representing_cartridge}
\label{sec:artifact-parameterization}

We parameterize $\ctxrep$ using a simplified version of prefix-tuning~\cite{li2021prefix}.
Specifically, we allocate a KV cache composed of \textit{trainable} key and value vectors $\mathbf{z}_\text{k}, \mathbf{z}_\text{v} \in \mathbb{R}^{p \times d}$. 
The size of the full $Z\in \mathbb{R}^{L \times p \times d \times 2}$ is controlled by the hyperparameter $p$.
The memory footprint of $Z$ is equivalent to a KV cache for a prompt with $p$ tokens.

In ICL, the KV cache for $\llm_\ctx(q)$ (where $\ctx$ is of length $n_\ctx$ and $\queries$ is of length $n_\queries$) would contain $n_\ctx + n_\queries$ key-value pairs, with the first $n_\ctx$ corresponding to $\ctx$ and the last $n_\queries$ corresponding to $\queries$:

\ifx\conference\neuripsconf 
\[
\begin{minipage}{0.50\textwidth}
\centering
\text{ICL KV Cache} \vspace{-1.5em} \\
\begin{align*}
\underbrace{(\mathbf{k}_1, \mathbf{v}_1), \dots, (\mathbf{k}_{n_\ctx}, \mathbf{v}_{n_\ctx})}_{\text{KV pairs for~}\ctx},
\underbrace{(\mathbf{k}_{n_\ctx + 1}, \mathbf{v}_{n_\ctx + 1})\dots}_{\text{KV pairs for } \query }
\end{align*}
\end{minipage}
\quad
\begin{minipage}{0.50\textwidth}
\centering
\text{\artifact KV Cache} \vspace{-1.5em} \\
\begin{align*}
\underbrace{(\mathbf{z}^\text{k}_1, \mathbf{z}^\text{v}_1), \dots, (\mathbf{z}^\text{k}_{p}, \mathbf{z}^\text{v}_{p})}_{\text{Trainable KV pairs in }Z}, 
\underbrace{(\mathbf{k}_{n_\ctx + 1}, \mathbf{v}_{n_\ctx + 1})\dots}_{\text{KV pairs for } \query }
\end{align*}
\end{minipage}
\]
\fi
\ifx\conference\arxivconf
{
\small
\[
\begin{minipage}{0.50\textwidth}
\centering
\text{ICL KV Cache} \vspace{-1.5em} \\
\begin{align*}
\underbrace{(\mathbf{k}[1], \mathbf{v}[1]), \dots, (\mathbf{k}[{n_\ctx}], \mathbf{v}[{n_\ctx}])}_{\text{KV pairs for~}\ctx},
\underbrace{(\mathbf{k}[{n_\ctx + 1}], \mathbf{v}[{n_\ctx + 1}])\dots}_{\text{KV pairs for } \query }
\end{align*}
\end{minipage}
\quad
\begin{minipage}{0.50\textwidth}
\centering
\text{\artifact KV Cache} \vspace{-1.5em} \\
\begin{align*}
\underbrace{(\mathbf{z}_\text{k}[1], \mathbf{z}_\text{v}[1]), \dots, (\mathbf{z}_\text{k}[p], \mathbf{z}_\text{v}[p])}_{ \text{Trainable KV pairs in }Z}, 
\underbrace{(\mathbf{k}[{1}], \mathbf{v}[{1}])\dots}_{\text{KV pairs for } \query }
\end{align*}
\end{minipage}
\]
}
\fi
\ifx\conference\icmlconf

\text{ICL KV Cache} \vspace{-1.5em} \\
\begin{align*}
\underbrace{(\mathbf{k}_1, \mathbf{v}_1), \dots, (\mathbf{k}_{n_\ctx}, \mathbf{v}_{n_\ctx})}_{\text{KV pairs for~}\ctx},
\underbrace{(\mathbf{k}_{n_\ctx + 1}, \mathbf{v}_{n_\ctx + 1})\dots}_{\text{KV pairs for } \query }
\end{align*}
\text{\artifact KV Cache} \vspace{-1.5em} \\
\begin{align*}
\underbrace{(\mathbf{z}^\text{k}_1, \mathbf{z}^\text{v}_1), \dots, (\mathbf{z}^\text{k}_{p}, \mathbf{z}^\text{v}_{p})}_{\text{Trainable KV pairs in }Z}, 
\underbrace{(\mathbf{k}_{n_p + 1}, \mathbf{v}_{n_p + 1})\dots}_{\text{KV pairs for } \query }
\end{align*}
\fi

To train a \artifact, we substitute the key-value pairs corresponding to $\ctx$ with $\ctxrep$, and directly optimize them by back-propagating the loss into the key and value vectors. 
\textbf{Critically, we freeze all parameters of the model, only training the key and value vectors in $Z$.} We discuss the choice of loss in \Cref{sec:method-objective}.

\vspace{-3mm}
\paragraph{Initialization} 
Prior work finds that optimizing a randomly initialized cache $\ctxrep$ is unstable and leads to degraded performance~\cite{li2021prefix}. 
Instead, these works initialize the trainable cache with a smaller dimensionality $d$ and then re-project it to the original dimension with an MLP. 
In contrast, we find that proper initialization of $\ctxrep$ allows us to directly optimize the full cache without reparametrization. 
Specifically, we initialize $\ctxrep$ to the KV cache corresponding to the first $p$ tokens of the corpus $\ctx$.
Alternatively, we could use a summary of the corpus or filter tokens using off-the-shelf prompt compression strategies~\cite{xiao2024duoattention}.
In \Cref{sec:results-ablations}, we show that our initializations lead to stable training and faster convergence than the random initialization.

\textit{Why this parameterization?} We note that the parameter-efficient fine-tuning literature provides other ways to augment an LLM with a set of additional parameters, in particular low-rank adaptation (LoRA)~\cite{li2021prefix,hu2022lora,lester2021power}. 
In \Cref{sec:results-ablations}, we perform a comprehensive comparison of \artifacts parameterized with prefix-tuning and LoRA.

\vspace{-2mm}
\subsection{Serving \artifacts}
\label{sec:artifact-serving}

A \artifact can be served efficiently with minimal changes to existing LLM inference servers~\cite{zheng2024sglang,kwon2023efficient,juravsky2025tokasaurus}. Because a \artifact is a KV cache, it can be loaded directly into the KV cache slots using existing mechanisms for handling cached prefixes. LLM inference servers are heavily optimized for managing distinct KV-caches for multiple users~\cite{ye2025flashinfer}, meaning \artifacts can be served at high throughput using existing inference servers. Decoding tokens with a \artifact is identical to serving a request with a prefix of length $p$ (the hyperparameter denoting the number of trainable tokens in the \artifact). This contrasts with other methods like LoRA, which require custom infrastructure to serve efficiently to multiple users~\cite{chen2024punica}. See Figure \ref{fig:micros} for the relationship between prefix length and throughput.

\vspace{-2mm}

\ifx\conference\neuripsconf
\begin{figure*}[t!]  
    \centering
    \includegraphics[width=\textwidth]{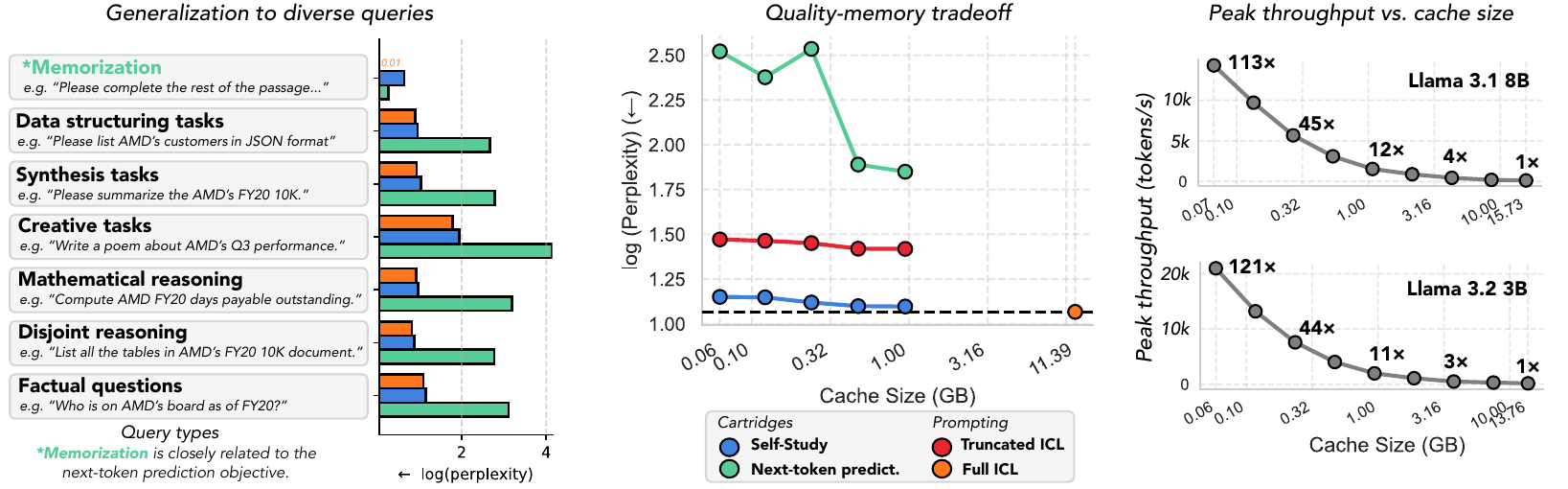}

    \caption{
        \textbf{\artifacts trained with \method balance the generality and memory consumption tradeoff.} 
        We compare four methods on the \genconvo dataset: \artifacts trained with next-token prediction over $\ctx$, \artifacts trained with \method, full ICL, and truncated ICL, a prompt compression method in which we truncate the $\ctx$ to the first $k$ tokens.   
        (\textbf{Left}) We evaluate on different slices from the \genconvo dataset. \artifacts trained with next-token prediction performs well on memorization queries, which resemble it's training distribution, but cannot generalize to other queries like the other methods.
        (\textbf{Center}) The $x$-axis measures the size of the KV cache in GB for the different methods.
        The $y$-axis shows log-perplexity on the \genconvo dataset averaged over the query types.
        (\textbf{Right}) Peak throughput (tokens/s) measured for different cache sizes for \llamathreeb and \llamaeightb with SGLang~\cite{zheng2024sglang} on an 1xH100 (See \Cref{app:results}).
    } 
    \label{fig:micros}
\end{figure*}

\fi

\section{\method: A self-supervised method for training \artifacts}
\label{sec:method}
In this section, we describe \method, a simple approach for training a \artifact $Z$ on any corpus of text. 
The design of \method is motivated by experiments showing how \artifacts trained with a simpler recipe fail to generalize to diverse user queries.

\ifx\conference\arxivconf
\begin{figure*}[t!]  
    \centering
    \includegraphics[width=\textwidth]{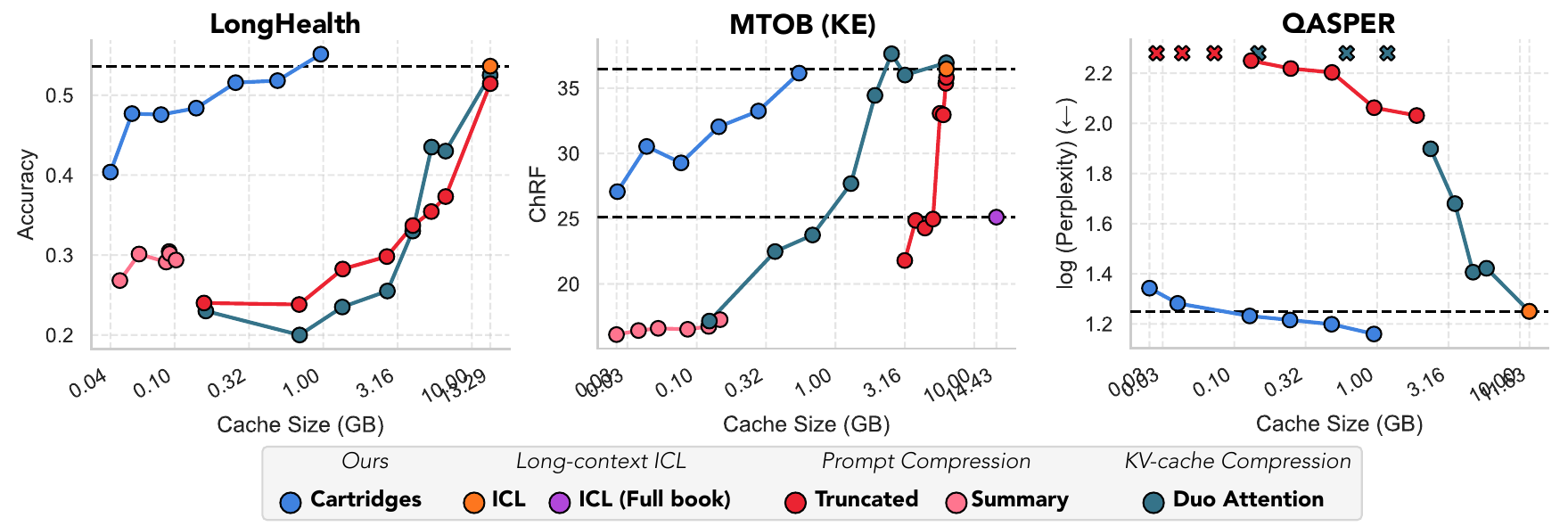}

    \caption{
        \artifacts \textbf{matches ICL quality with lower memory costs.} 
        We measure \llamathreeb response quality ($y$-axis) against KV cache memory ($x$-axis) for different methods, at different KV cache sizes. The dashed line marks the quality of standard ICL. 
        }
    \label{fig:tradeoff-within}
    \vspace{-2mm}
\end{figure*}

\fi

\vspace{-2mm}
\paragraph{Motivating observations}
The naive method for constructing a \artifact would be to fine-tune the parameters of $Z$ with the next token prediction objective on the corpus text directly. We show results experimenting with this approach in \Cref{fig:micros}, where we evaluate on a dataset derived from FinanceBench~\cite{islam2023financebench}, which we refer to as \genconvo (see \Cref{app:datasets} for details). \genconvo contains multiple types of questions (\eg synthesis, reasoning). 
We find that the naïve next-token prediction approach can memorize with near perfect perplexity (\Cref{fig:micros} left), while consuming $107\times$ less memory than ICL (\Cref{fig:micros} center). 
However, generalization to other slices is poor, as shown in \Cref{fig:micros}. We seek a training objective that allows the responses from a model that uses the \artifact to generalize to a diverse set of user queries, resembling ICL.

Motivated by these observations, we describe a synthetic data generation recipe in \Cref{sec:method-data} and a context-distillation objective in \Cref{sec:method-objective}. As we show in \Cref{fig:micros}, \artifacts trained with this approach can generate responses to many types of queries that match the quality of queries generated with ICL.
See \Cref{fig:banner} for a visualization of the \artifact approach.

\ifx\conference\icmlconf
\begin{figure*}[t!]  
    \centering
    \includegraphics[width=\textwidth]{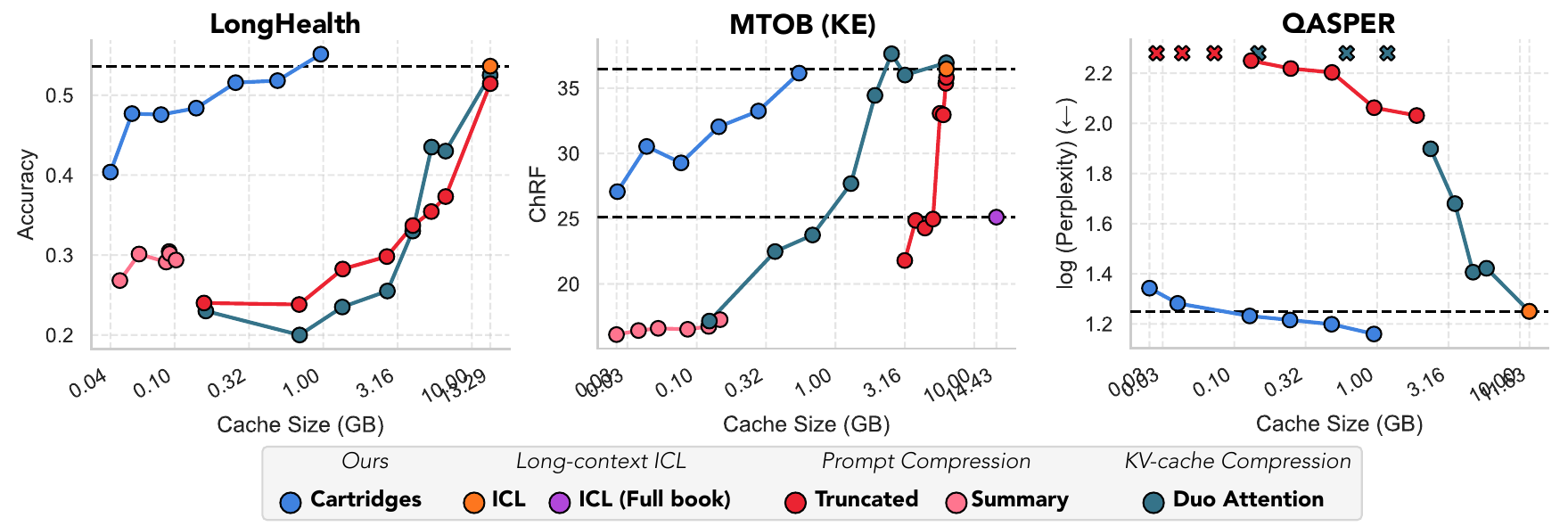}

    \caption{
        \artifacts \textbf{matches ICL quality with lower memory costs.} 
        We measure \llamathreeb response quality ($y$-axis) against KV cache memory ($x$-axis) for different methods, at different KV cache sizes. The dashed line marks the quality of standard ICL. 
        }
    \label{fig:tradeoff-within}
    \vspace{-2mm}
\end{figure*}

\fi

\vspace{-2mm}
\subsection{Self-supervised synthetic data to avoid overfitting}
\label{sec:method-data}

Towards training general \artifacts, we propose using LLM generated synthetic data to generate our training dataset $\mathcal{D}_{\text{train}}$. 

\vspace{-2mm}
\paragraph{Overall synthetic data pipeline} Our overall pipeline puts information from the corpus $\ctx$ in context and prompts the model to have a conversation with itself about the corpus to generate the synthetic query-response pairs as shown in \Cref{alg:synthetic-generation}. We represent the concatenation of two vectors with $x \oplus y$.

\begin{algorithm}
    
    \caption{\method: Data Generation}
    \textbf{Input:} $\ctx$ : \texttt{Corpus}, $\llm$ : \texttt{Model}  \\ 
    \textbf{Output:} $\{\mathbf{a}_1, \mathbf{b}_1, \dots, \mathbf{a}_k, \mathbf{b}_k\}$ :  \texttt{Convo}
    \begin{algorithmic}[1]
    \label{alg:synthetic-generation}

    \State $\subctx \gets$ \texttt{chunk}($\ctx$)  \Comment{\textbf{(1)} Get a \textbf{subcorpus} of $\ctx$ that fits in the context window}
    \State $\seed \gets$ \texttt{get\_seed\_prompt}()  \Comment{\textbf{(2)} Get a prompt to \textbf{seed} the first message from $A$}
    \For{$i = 1$ to $k$} \Comment{\textbf{(3)} Sample a \textbf{conversation} with $k$ back and forths}
        \State $\mathbf{a}_i \sim \llm( \cdot \mid \subctx \oplus \seed \oplus \mathbf{a}_{1} \oplus \dots \oplus \mathbf{b}_{i-1})$  \Comment{\textbf{(3.1)} Sample $A$'s message with $\subctx$ and $\seed$ in context}
        \State $\mathbf{b}_i \sim \llm( \cdot \mid \subctx \oplus \mathbf{a}_{1} \oplus \dots \oplus \mathbf{b}_{i-1} \oplus \mathbf{a}_{i})$  \Comment{\textbf{(3.2)} Sample $B$'s message with $\subctx$ in context}
    \EndFor
    \State \textbf{return} $ \{\mathbf{a}_1, \mathbf{b}_1, \dots, \mathbf{a}_k, \mathbf{b}_k\}$ 
    \end{algorithmic}
    \end{algorithm}

The conversation is generated by iteratively sampling generations from two LLM participants $A$ and $B$ (which are the same model). We maintain two different conversation histories: $A$'s starts with a \textit{user} message containing a seed prompt $s$ (\eg \textit{``Please start a conversation by asking a question about the document above."}) followed by alternating \textit{assistant} and \textit{user} messages from $A$ and $B$, respectively. $B$'s conversation history does not include the seed prompt and contains the same messages as $A$'s but with the roles of $A$ and $B$ swapped. Both have the subcorpus $\subctx$ in the system prompt. To build a training dataset, we sample $\numtrain$ independent conversations and concatenate the messages from $A$ and $B$ into a single sequence of tokens:
\begin{equation}
	\mathcal{D}_\text{train} = 
	\{
	\mathbf{x}^{(j)} = 
	\mathbf{a}_1^{(j)} \oplus 
	\mathbf{b}_1^{(j)} \oplus 
	\mathbf{a}_2^{(j)} \oplus 
	\mathbf{b}_2^{(j)} \oplus 
	\dots \oplus 
	\mathbf{a}_k^{(j)} \oplus 
	\mathbf{b}_k^{(j)}
	\}_{j=1}^{\numtrain}
    \label{eq:dataset}
\end{equation}
where each $\mathbf{x}^{(j)}$ is a concatentation of the messages. 
Note that all of the datasets on which we evaluate in the main paper involve a single-turn. So, we set $k=1$, generating a synthetic conversation with one user message and one assistant message.

Note that the \texttt{chunk} and \texttt{get\_seed\_prompt} functions expose two different ways to control the data distribution of the synthetic data. 
We find that these two design decisions are critical for training high quality \artifacts with \method.

\vspace{-2mm}
\paragraph{Chunking}  We use short subcorpora $\tilde{c}$ (between 512 and 4096) tokens to let the LLM focus on different parts of the corpus when generating data.
This is motivated by observations in prior work~\cite{liu2024lost, narayan2025minions}.
Furthermore, chunking also allows us to train \artifacts on corpora longer than the model's context window. 

\vspace{-2mm}
\paragraph{Seed prompts}
Instead of using just one seed prompt, we curate a list of five different seed prompt types:     \textit{structuring}, \textit{summarization},
    \textit{question}, \textit{use cases}, and
    \textit{creative}.
The full list of seed prompts used in our experiments is provided in \Cref{app:method}. 
Critically, in all our experiments the seed prompts are \textbf{generic}: they do not mention anything related to the specifics of the corpora we evaluated (\eg no mention of translation for MTOB or medical terms for LongHealth). 
We use the same set of seed prompts in all of our main results.
In \Cref{sec:results-ablations}, we ablate the use of diverse seed prompts and find that it improves performance over a single generic seed prompt by up to $4.8$ accuracy points ($43.6 \rightarrow 48.4$ on \longhealth). 

\ifx\conference\icmlconf
\begin{figure*}[t]  
    \centering
    \includegraphics[width=\textwidth]{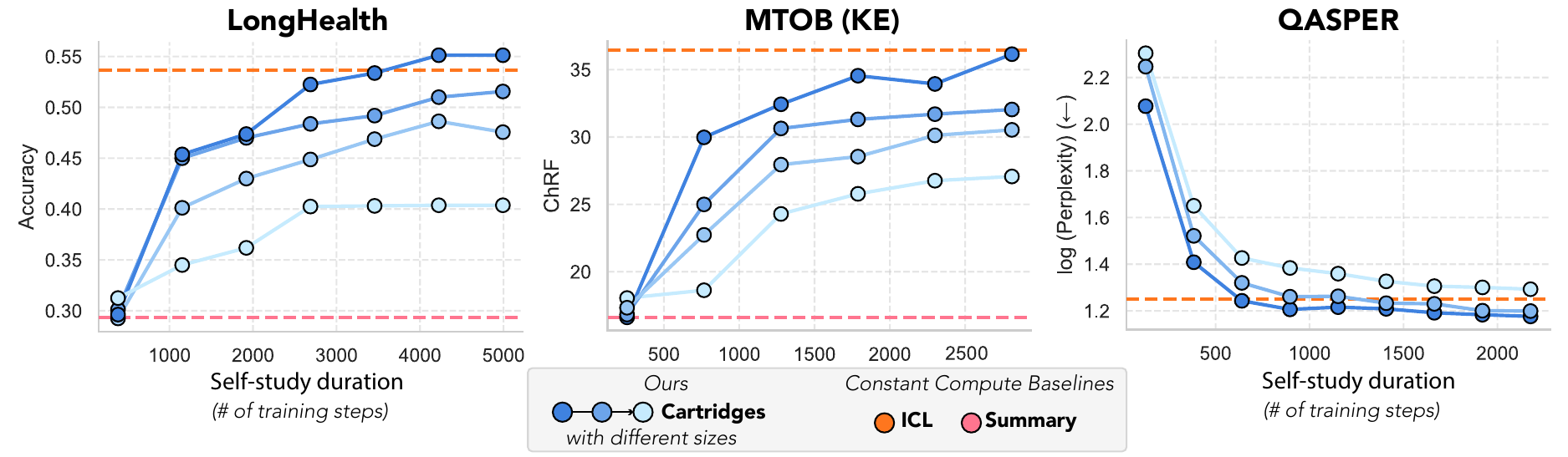}

    \caption{\textbf{Scaling \method compute.} These plots show how quality improves as we scale the training compute with \method. 
        In all plots, the $x$-axis shows the total number of global training steps with batch size 64 and maximum sequence length 1024. 
        No synthetically generated data is reused (\ie training proceeds for one epoch).
        Curves are provided for \artifacts of varying sizes ($p \in \{128, 512, 2048, 8192\}$).
        (\textbf{Left}) The $y$-axis shows accuracy on \longhealth~\cite{adams2024longhealth} with \llamaeightb.
        (\textbf{Middle}) The $y$-axis shows the chrF on \mtob~\cite{tanzer2023benchmark} with \llamathreeb.
        (\textbf{Right}) The $y$-axis shows log-perplexity (lower is better) on \qasper~\cite{dasigi2021dataset} with \llamathreeb.
    }
    \label{fig:scaling-indexing}

\end{figure*}
\fi

\vspace{-2mm}
\subsection{\method context-distillation objective}
\label{sec:method-objective}

Given a fine-tuning dataset $\mathcal{D}_\text{train}$, we adapt standard techniques from the model distillation literature~\cite{kim2016sequence,snell2022learning,kujanpaa2024knowledge}. We let $\llm(\cdot | \mathbf{x})$ denote the next token distribution given some input text $\mathbf{x}$. Our \textit{teacher} is the model with the subcorpus, $\subctx$, in context $\llm( \cdot | \subctx)$ and our \textit{student} is the same model adapted with a trainable cache   $\llm_{\ctxrep}( \cdot)$.
We use a classic distillation objective~\cite{hinton2015distilling} that minimizes the KL-divergence between the teacher and student next-token distributions over a sequence of tokens $\mathbf{x}$ and the corresponding subcorpus used to generate them $\subctx$.
\vspace{-2mm}
\begin{equation}
    \underset{\ctxrep}{\arg\min} \quad
        \sum_{(\mathbf{x}, \subctx) \in \mathcal{D}_\text{train}}  
            \sum_{i=1}^{|\mathbf{x}|}  
                D_{\text{KL}}\bigg(
                    \llm( \cdot | \subctx \oplus \mathbf{x}[:i]) \quad || \quad \llm_{\ctxrep}( \cdot | \mathbf{x}[: i])   
                \bigg)
\end{equation}

In \Cref{app:results}, ablate the use of the context-distillation objective and show that improves accuracy when controlling for the amount of synthetic data (\eg $3.7$ accuracy points on \longhealth).

\vspace{-2mm}
\section{Results}
\label{sec:results}

\ifx\conference\neuripsconf
\begin{figure*}[t!]  
    \centering
    \includegraphics[width=\textwidth]{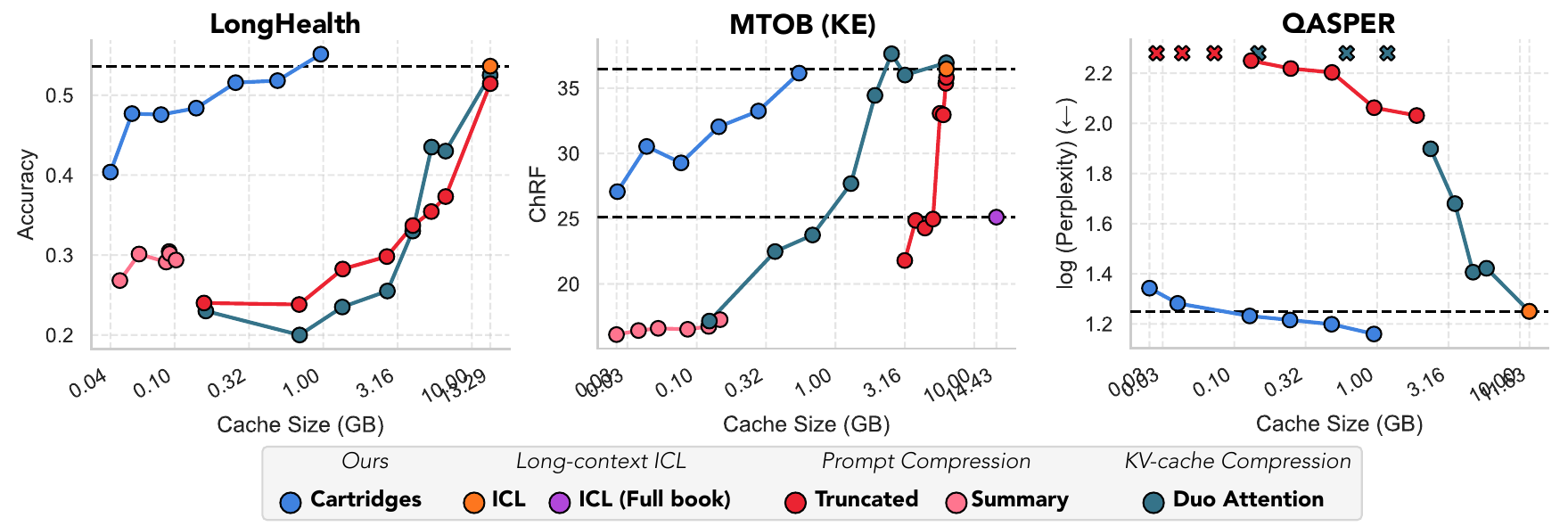}

    \caption{
        \artifacts \textbf{matches ICL quality with lower memory costs.} 
        We measure \llamathreeb response quality ($y$-axis) against KV cache memory ($x$-axis) for different methods, at different KV cache sizes. The dashed line marks the quality of standard ICL. 
        }
    \label{fig:tradeoff-within}
    \vspace{-2mm}
\end{figure*}

\fi
\ifx\conference\arxivconf
\begin{figure*}[t]  
    \centering
    \includegraphics[width=\textwidth]{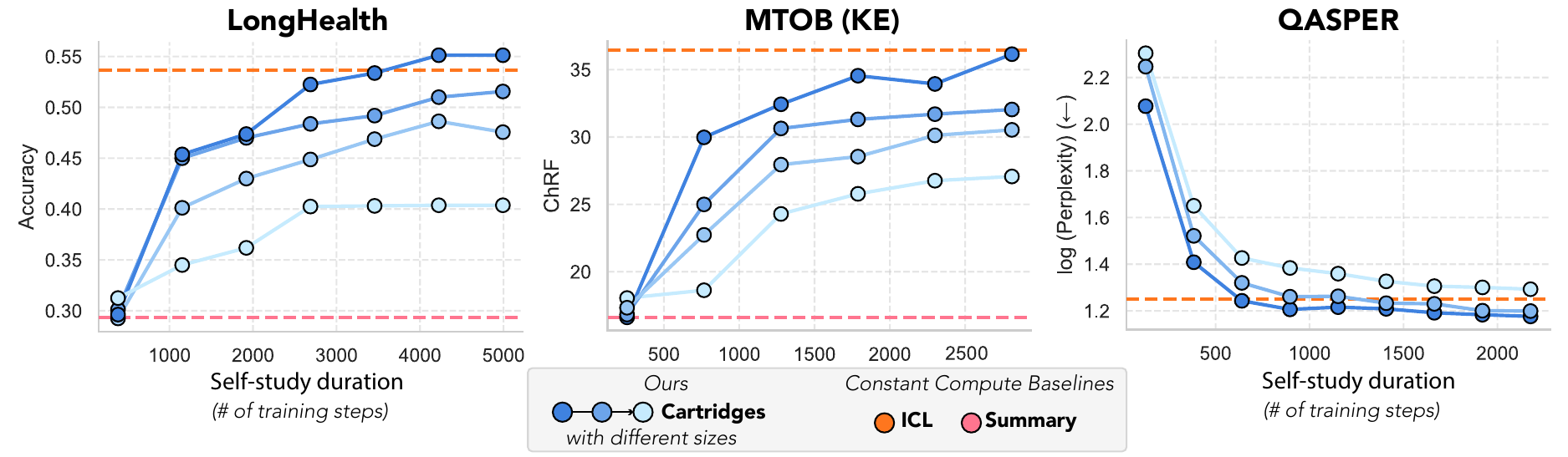}

    \caption{\textbf{Scaling \method compute.} These plots show how quality improves as we scale the training compute with \method. 
        In all plots, the $x$-axis shows the total number of global training steps with batch size 64 and maximum sequence length 1024. 
        No synthetically generated data is reused (\ie training proceeds for one epoch).
        Curves are provided for \artifacts of varying sizes ($p \in \{128, 512, 2048, 8192\}$).
        (\textbf{Left}) The $y$-axis shows accuracy on \longhealth~\cite{adams2024longhealth} with \llamaeightb.
        (\textbf{Middle}) The $y$-axis shows the chrF on \mtob~\cite{tanzer2023benchmark} with \llamathreeb.
        (\textbf{Right}) The $y$-axis shows log-perplexity (lower is better) on \qasper~\cite{dasigi2021dataset} with \llamathreeb.
    }
    \label{fig:scaling-indexing}

\end{figure*}
\fi

We describe experiments evaluating the effectiveness of \artifacts trained with \method in various long-context scenarios. 
Our results support the following claims. 
\textbf{First}, \artifacts trained with \method can match or outperform ICL while maintaining generality and reducing serving costs (\Cref{sec:results-within}). 
\textbf{Second}, \method is effective on corpora longer than the context window of the LLM (\Cref{sec:results-extending}).
\textbf{Third}, when we concatenate two different \artifacts without any joint training, the model can respond to queries requiring information from both \artifacts (\Cref{sec:results-composition}). 
Finally, we include ablations to assess the relative benefits of different aspects of \method and \artifacts (\Cref{sec:results-ablations}).

\vspace{-2mm}
\paragraph{Datasets} We study datasets consisting of diverse $(\query, \resp)$ pairs about a single long document. Across datasets, $\ctx$ ranges between 100k and 484k tokens. Our datasets are drawn from popular long-context benchmarks, with some used as-released and others modified to meet this structure. These include: \longhealth~\cite{adams2024longhealth}, \mtob~\cite{tanzer2023benchmark}, and QASPER~\cite{dasigi2021dataset}. We evaluate LLM response quality using accuracy for \longhealth, log perplexity for QASPER, and character n-gram f-score (chrF) for MTOB~\cite{tanzer2023benchmark, popovic2015chrf}. Because each dataset effectively consists of a ``single'' document, we train a single \artifact per dataset and evaluate it on the queries response pairs $(\query, \resp)$. \Cref{app:datasets} provides further details. 

\ifx\conference\icmlconf
\begin{figure*}[t]  
    \centering
    \includegraphics[width=\textwidth]{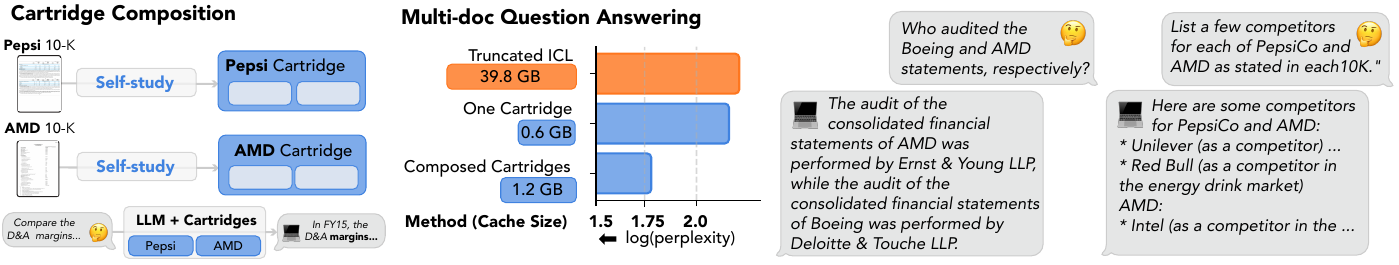}

    \caption{
        \textbf{\artifact Composition.}  
        (\textbf{Left}) Illustration of \artifact composition, where two independently trained \artifacts (one for a Pepsi 10-K and one for an AMD 10-K) are concatenated without any additional training.
        (\textbf{Middle}) We evaluate composition on a dataset of multi-document questions requiring information in two different $\approx$100k token documents with \llamathreeb~(see \Cref{app:datasets}). 
        The $x$-axis shows log-perplexity (lower is better) on gold-standard answers.
        We compare \artifact composition with an (a) ICL baseline where we truncate the document to fit in the 128k token context length and (b) an \artifact baseline where we only include the \artifact for one of the documents.
        (\textbf{Right}) Examples of responses to multi-document questions using composed cartridges.
    }
    \label{fig:composition}
\vspace{-2mm}
\end{figure*}

\fi

\vspace{-2mm}
\subsection{Pushing the quality/cost tradeoff frontier}
\label{sec:results-within}
We assess how \artifacts produced with \method fare in quality and memory consumption against baselines for \longhealth and QASPER on \llamathreeb. 
For both datasets, $\ctx$ fits within the model context window ($128$k tokens). 
We compare to traditional ICL, two prompt compression baselines (prompt truncation and prompt summarization using GPT-4o  \cite{openai2024gpt4ocard}), and a state-of-the-art KV cache compression baseline (Duo Attention~\cite{jiang-etal-2023-llmlingua,xiao2024duoattention}). We evaluate memory use in terms of KV cache size: the size of the KV cache for the ICL model and prompt compression methods, the size of the \artifact, and the size of the compressed KV cache for KV cache compression methods like DuoAttention.

\Cref{fig:tradeoff-within} presents our main results. On both \longhealth and QASPER, we find cache sizes at which \artifacts outperforms ICL. Compared against ICL, \artifacts offers substantial memory savings at comparable performance: up to $10\times$ for \longhealth,  and up to $100\times$ for QASPER. In contrast, compression baseline methods see performance degradations at compression factors as low as $2\times$. Crucially, the small memory footprint of \artifacts allows for much higher peak throughput (tokens/s). 
As \Cref{fig:micros} (right) shows, cache sizes which match performance of ICL allow for almost $26\times$ higher throughput.

\ifx\conference\neuripsconf
\begin{figure*}[t]  
    \centering
    \includegraphics[width=\textwidth]{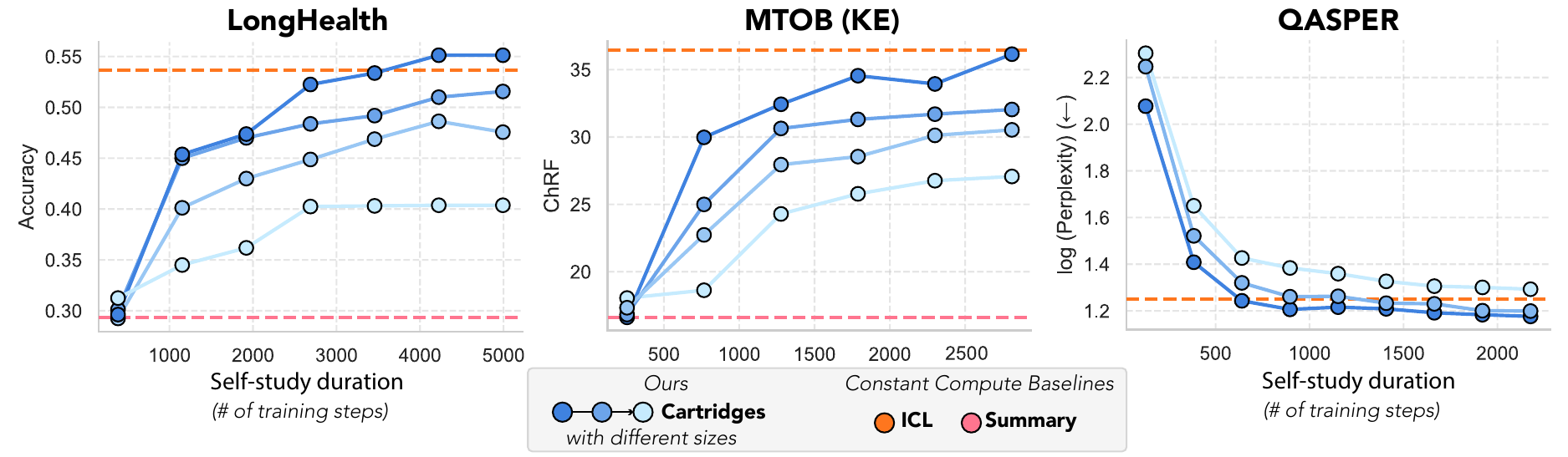}

    \caption{\textbf{Scaling \method compute.} These plots show how quality improves as we scale the training compute with \method. 
        In all plots, the $x$-axis shows the total number of global training steps with batch size 64 and maximum sequence length 1024. 
        No synthetically generated data is reused (\ie training proceeds for one epoch).
        Curves are provided for \artifacts of varying sizes ($p \in \{128, 512, 2048, 8192\}$).
        (\textbf{Left}) The $y$-axis shows accuracy on \longhealth~\cite{adams2024longhealth} with \llamaeightb.
        (\textbf{Middle}) The $y$-axis shows the chrF on \mtob~\cite{tanzer2023benchmark} with \llamathreeb.
        (\textbf{Right}) The $y$-axis shows log-perplexity (lower is better) on \qasper~\cite{dasigi2021dataset} with \llamathreeb.
    }
    \label{fig:scaling-indexing}

\end{figure*}
\fi
\ifx\conference\arxivconf
\begin{figure*}[t!]  
    \centering
    \includegraphics[width=\textwidth]{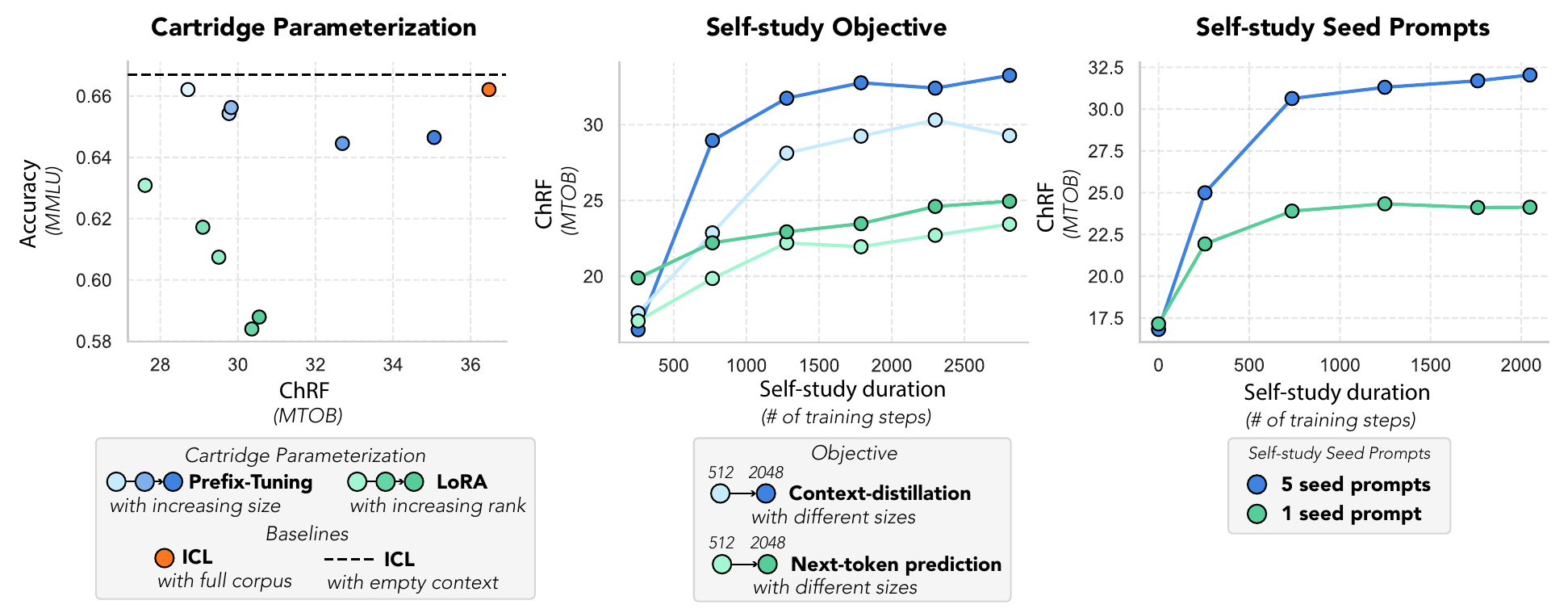}

    \caption{
        \textbf{Ablating \artifact and \method design choices.} 
        Ablations were performed on the \mtob dataset (see \Cref{app:results} for full ablation experiments). 
        (\textbf{Left}) We train \artifacts using two different parameterizations: simplified prefix-tuning (as described in \Cref{sec:artifact-parameterization}) and low-rank adaptation (LoRA)~\cite{hu2022lora}. The $x$-axis shows accuracy on MMLU and the $y$-axis shows accuracy on the target dataset. Each point represents a different \artifact size.
        \textbf{Center} We train \artifacts with \method using two loss functions: a next token prediction loss (green) and a distillation loss (blue). The $x$ axis is the number of training steps, and the $y$ axis is accuracy. Each hue represents a different \artifact size.
        (\textbf{Right}) We generate synthetic data according to \Cref{alg:synthetic-generation} and ablate the choice of seed prompts sampled on Line 2. We consider two approaches: using a single, broad seed prompt (Green) or randomly sampling one of five different types of seed prompts (Blue). The $x$ axis is the number of training steps, and the $y$ axis is accuracy.
        }
    \label{fig:ablations}
    \vspace{-2mm}
\end{figure*}

\fi

We also observe that \artifact performance scales as we increase the amount of compute used in self-study: the longer an \artifact is trained, the greater task performance. \Cref{fig:scaling-indexing} plots the performance for differentially sized \artifacts as a function of the number of training steps. Across all sizes, we observe a steady positive correlation between performance and compute. 

\vspace{-2mm}
\subsection{Extending the effective context window}
\label{sec:results-extending}

We evaluate whether \method allows us to accurately process corpora that exceed the context window length. 
To study this, we consider the MTOB dataset, and \llamaeightb, which has a context window of $128$k tokens. 
MTOB provides two different long documents: a full $484$k token latex textbook and a shorter $60$k token version, which was manually-curated by the dataset authors to exclude content not relevant to the translation task. 
Even though the $484$k textbook is $356$k tokens \textit{longer} than \llamaeightb's context window length, we can produce a \artifact for the full textbook using the chunking strategy of \method.

\Cref{fig:tradeoff-within} (middle plot) shows the performance of \artifacts of various sizes trained with \method. 

As a point of comparison, we provide the results for KV cache baseline methods on the smaller $60$k token textbook, and also include ICL on a truncated version of the long textbook. Like above, we observe that \artifact can match the performance of ICL on the hand-curated $60$k token version, while requiring substantially less memory and only having access to the $484$k token version, which exceeds the context window of \llamaeightb.
\artifacts also outperform competitive baselines at every KV cache size, by up to 11.0 chrF points. 

\vspace{-2mm}
\subsection{Ablating \method design choices}
\label{sec:results-ablations}

We perform ablations to study different aspects of \method and \artifact parameterization. 
We provide full results in Appendix \ref{app:results} and highlight key findings here and in \Cref{fig:ablations}.

\paragraph{\artifact Parameterization} In \Cref{sec:artifact-parameterization}, we discuss how we parameterize the \artifact with a trainable KV cache, which is equivalent to a simplified version of prefix tuning~\cite{li2021prefix}.
There are a number of other ways we could parameterize the \artifact, notably low-rank adaptation (LoRA), an extremely popular parameter effcient fine-tuning method~\cite{hu2022lora}. 

We compare the prefix-tuning parameterization with LoRA (see \Cref{app:results-parameterization} for full results). 
First, we find that the prefix-tuning parameterization is more effective than a memory-matched LoRA parameterization on queries related to the corpus. For example, with \artifacts of size $\sim0.6$ GB on \mtob, prefix-tuning outperforms LoRA by $4.5$ ChRF points. (See \Cref{fig:parameterization} for results on \longhealth and \qasper.)
Even more interesting is the gap between these parameterizations on queries unrelated to the document like MMLU~\cite{hendrycks2020measuring}.
When using a LoRA parameterization, we find that MMLU accuracy drops precipitously (from $54.7$ to $45.3$) as we increase the \artifact size (from 0.15 GB to 1.06 GB). In contrast, with prefix-tuning, the accuracy drops much less rapidly (from $54.7$ to $54.3$) as we increase the size (from 0.15 GB to 0.96 GB).
See \Cref{fig:parameterization} for plots illustrating these findings on \longhealth, \qasper, and \mtob.
We also show that freezing the attention sink (the first token in the key and value vectors) improves training stability (\Cref{fig:freeze}). 

\paragraph{\artifact Initialization} We compare three different strategies for initializing the KV cache when using the prefix-tuning parameterization: (1) random vectors (from a component-wise standard normal distribution), (2) key and value vectors of random tokens, and (3) key and value vectors of the first $p$ tokens of the corpus.
We find that initializing with key and value vectors of actual tokens (as opposed to random vectors) is critical for achieving ICL-level performance. 
On \longhealth, random vectors achieve an accuracy of $29.9\%$ while key and value vectors of random tokens achieve an accuracy of $51.3\%$.
Initializing with the first $p$ tokens provides an additional improvement of $4$ percentage points to $55.3\%$.
In the original prefix-tuning paper, the authors show that initializing from tokens improves performance when performing supervised fine-tuning on very small datasets~\cite{li2021prefix}. Our results extend this finding to \method, where we train on large synthetic datasets.

\paragraph{\method Seed Prompts} Next, we ablate the choice of \textit{seed prompts} (see Line 2 of \Cref{alg:synthetic-generation}). 
We compare two approaches: (1) always using the same seed prompt (\textit{``Please generate a single chat message to begin a conversation about the information in the corpus. Ask a question about the corpus or make a request."}) and (2) randomly sampling one of five different types  of seed prompts (\eg structuring, summarization; see full list in \Cref{app:method-data}). Note even with the latter approach, the seed prompts are generic: the same set of seed prompts are used for all corpora. 
On \mtob, we find that using this small set of seed prompts improves over the single seed prompt by 7.9 ChRF points ($24.1 \rightarrow 32.0$; see \Cref{fig:ablations} Left). 
On \longhealth, the improvement is $4.8$ accuracy points ($43.6 \rightarrow 48.4$ on \longhealth; see \Cref{fig:seeds}). 
Interestingly, on \qasper we do not see any significant benefit from using the diverse seed prompts. This is perhaps because, compared to \longhealth and \mtob, the queries in \qasper are less reasoning intensive.

\paragraph{\method Objective} Finally, we evaluate the importance of the context distillation objective (defined in \Cref{sec:method-objective}). 
Using the same \method synthetic data for both objectives, we compare the context-distillation objective with a simpler next-token prediction objective.
%
On \mtob, we find that using a context distillation objective on the synthetic conversation data improves ChRF by $8.6$ points ($24.9 \rightarrow 33.5$; see \Cref{fig:logits} Center).
We also see improvements on \longhealth and \qasper (see \Cref{fig:logits}).

\vspace{-2mm}
\subsection{Composing \artifacts}
\label{sec:results-composition}

\ifx\conference\neuripsconf
\begin{figure*}[t]  
    \centering
    \includegraphics[width=\textwidth]{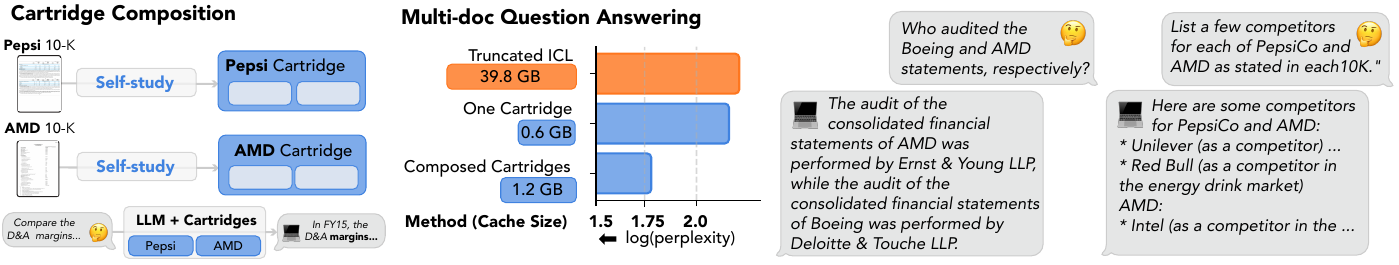}

    \caption{
        \textbf{\artifact Composition.}  
        (\textbf{Left}) Illustration of \artifact composition, where two independently trained \artifacts (one for a Pepsi 10-K and one for an AMD 10-K) are concatenated without any additional training.
        (\textbf{Middle}) We evaluate composition on a dataset of multi-document questions requiring information in two different $\approx$100k token documents with \llamathreeb~(see \Cref{app:datasets}). 
        The $x$-axis shows log-perplexity (lower is better) on gold-standard answers.
        We compare \artifact composition with an (a) ICL baseline where we truncate the document to fit in the 128k token context length and (b) an \artifact baseline where we only include the \artifact for one of the documents.
        (\textbf{Right}) Examples of responses to multi-document questions using composed cartridges.
    }
    \label{fig:composition}
\vspace{-2mm}
\end{figure*}

\fi
\ifx\conference\arxivconf
\begin{figure*}[t]  
    \centering
    \includegraphics[width=\textwidth]{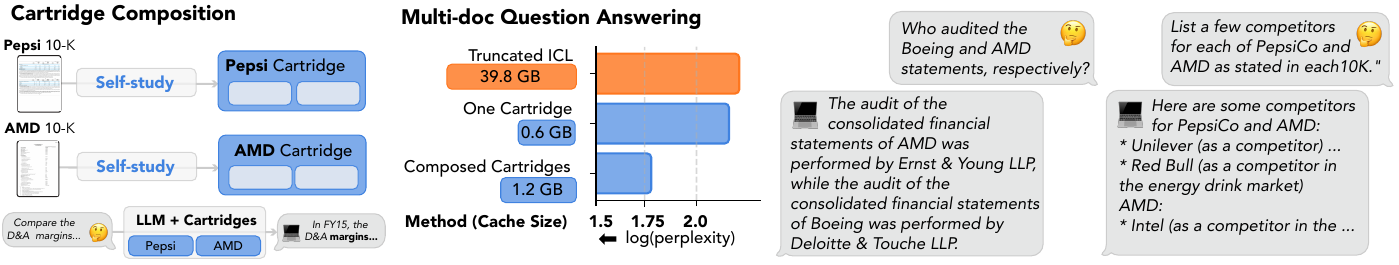}

    \caption{
        \textbf{\artifact Composition.}  
        (\textbf{Left}) Illustration of \artifact composition, where two independently trained \artifacts (one for a Pepsi 10-K and one for an AMD 10-K) are concatenated without any additional training.
        (\textbf{Middle}) We evaluate composition on a dataset of multi-document questions requiring information in two different $\approx$100k token documents with \llamathreeb~(see \Cref{app:datasets}). 
        The $x$-axis shows log-perplexity (lower is better) on gold-standard answers.
        We compare \artifact composition with an (a) ICL baseline where we truncate the document to fit in the 128k token context length and (b) an \artifact baseline where we only include the \artifact for one of the documents.
        (\textbf{Right}) Examples of responses to multi-document questions using composed cartridges.
    }
    \label{fig:composition}
\vspace{-2mm}
\end{figure*}

\fi

We evaluate if independently trained \artifacts can be \textit{composed} in order to serve queries about two different 
corpora (see \Cref{fig:composition}, Left). 
We train \artifacts across sizes $\{512, 1024, 2048, 4096\}$ and long 10-K documents from AMD, Pepsi, AMEX, and Boeing~\cite{islam2023financebench}. 
For each pair of \artifacts pairwise (6 pairs per cache size), we evaluate using a dataset of \textit{multi-document questions}, i.e., requiring information from both 10-Ks. 
Surprisingly, we find composition not only leads to coherent LLM generations \textit{off-the-shelf without any re-training} (\Cref{fig:composition}, Right), but also substantially outperforms the use of a single \artifact (\ie for only AMD) or ICL (which struggles due to context length limits) (\Cref{fig:composition}, Center) on the multi-document questions.

\vspace{-2mm}
\section{Discussion and conclusion}
We propose \artifacts as an alternative to ICL for settings where many different user messages reference the same large corpus of text.  
We demonstrate across a diverse set of language model workloads that, when trained via \method, they match ICL's response quality while substantially reducing memory consumption  ($38.6\times$ memory reduction across our evaluations) and increasing peak throughput ($26.4\times$ higher tokens per second). \artifacts are simple to train, composable, and compatible with existing LLM serving infrastructure.

However, compared with ICL, \method is not without limitations. 
Using \method to produce a KV-cache is much more costly than simply running standard ICL pre-fill. 
With our unoptimized implementation, training an ICL-quality \artifact takes $\sim30$ minutes on a single $8\times$H100 node (for \llamaeightb)
So our work does not provide a drop-in replacement for ICL, but rather demonstrates one way to tradeoff increased compute for reduced memory when constructing a KV-cache. 
This tradeoff is extremely advantageous in many settings: users often issue many queries over the same corpus and \method can be trained offline on idle or underutilized compute (\eg at night when user load is low~\cite{jaiswal2025serving,goel2025niyama}). 
Furthermore, there is ample room for optimizations (\eg improved shared-prefix attention kernels~\cite{dao2022flashattention,ye2025flashinfer, juravsky2024hydragenhighthroughputllminference}) that would make \method training procedure more efficient.

Looking forward, we envision \artifacts enabling a broad class of context-aware AI applications that are intractable with ICL today, from medical assistants that know a patient's full medical history to LLM-powered IDEs that understand entire codebases.

\paragraph{Acknowledgments} We thank Jordan Juravsky, Dan Biderman, Tri Dao, Bradley Brown, Mayee Chen, Avanika Narayan, Avner May, Bill Mark, Benjamin Spector, Roberto Garcia, Quinn Mcintyre, Yasa Baig, Geoff Angus, Kelly Buchanan, Mert Yuksekgonul, Eric Nguyen, Eric Wu, Kevin Wu, Owen Dugan, Jon Saad-Falcon, Simon Guo and the entire Zou, Hazy, and Scaling Intelligence research labs for helpful discussions and feedback.
We gratefully acknowledge Modal, Prime Intellect, Voltage Park, and Together AI for providing the GPUs to support for this work.
We gratefully acknowledge the support of NIH under No. U54EB020405 (Mobilize), NSF under Nos. CCF2247015 (Hardware-Aware), CCF1763315 (Beyond Sparsity), CCF1563078 (Volume to Velocity), and 1937301 (RTML); US DEVCOM ARL under Nos. W911NF-23-2-0184 (Long-context) and W911NF-21-2-0251 (Interactive Human-AI Teaming); ONR under Nos. N000142312633 (Deep Signal Processing); Stanford HAI under No. 247183; NXP, Xilinx, LETI-CEA, Intel, IBM, Microsoft, NEC, Toshiba, TSMC, ARM, Hitachi, BASF, Accenture, Ericsson, Qualcomm, Analog Devices, Google Cloud, Salesforce, Total, the HAI-GCP Cloud Credits for Research program,  the Stanford Data Science Initiative (SDSI), members of the Stanford SEAMS project: IBM and Felicis, as well as members of the Stanford DAWN project: Meta, Google, and VMWare. SE is supported by the NSF Graduate Research Fellowship Program. AR's research is supported by NSF grant CCF\#2247014. The U.S. Government is authorized to reproduce and distribute reprints for Governmental purposes notwithstanding any copyright notation thereon. Any opinions, findings, and conclusions or recommendations expressed in this material are those of the authors and do not necessarily reflect the views, policies, or endorsements, either expressed or implied, of NIH, ONR, or the U.S. Government. 

\paragraph{Contributions}
SE and RE conceived of \artifacts and \method. SE, RE, and SA designed the method, implemented the experiments, wrote the manuscript, and contributed equally to the project. NG made substantial contributions to the structure of the project and the final manuscript. EL and DZ implemented and ran experiments and made meaningful contributions to the manuscript. 
WT implemented the LoRA baselines. 
DZ and AR led the theoretical analysis. 
AR, JZ, AM, and CR supervised the project. 

\ifx\conference\neuripsconf
    \bibliographystyle{plain}
    
\else\ifx\conference\icmlconf
    \bibliographystyle{icml2025}
\else\ifx\conference\arxivconf
    \bibliographystyle{plain}
\fi
\fi
\bibliography{references}

\clearpage

\appendix
\begin{figure*}[htbp]  
    \centering
    \includegraphics[width=\textwidth]{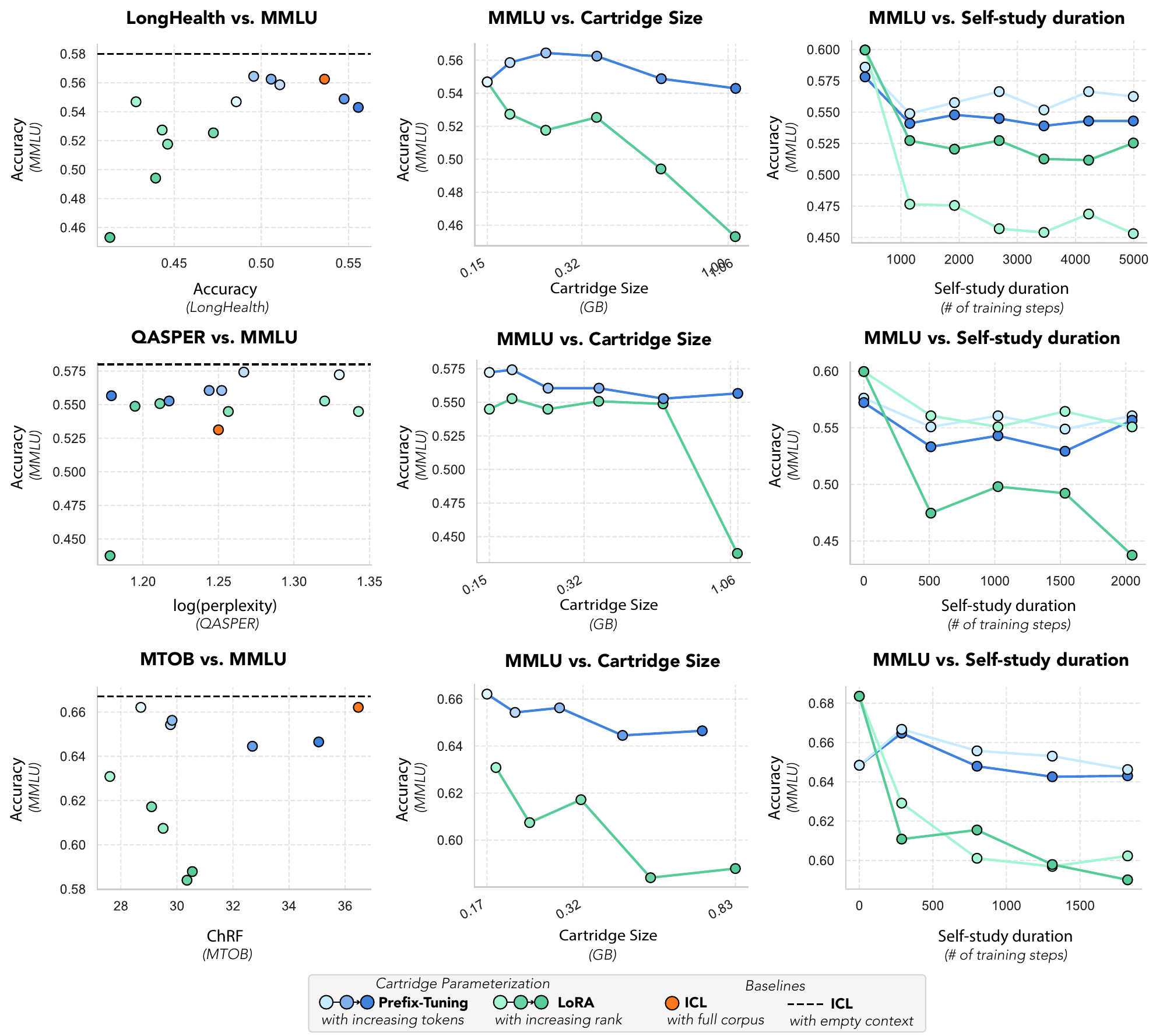}

    \caption{
        \textbf{Comparing \artifact parameterizations.} 
        We train \artifacts using \method on the corpora from \longhealth (Top), \qasper (Middle), and \mtob (Bottom) using two different parameterizations: simplified prefix-tuning (as described in \Cref{sec:artifact-parameterization}) and low-rank adaptation (LoRA)~\cite{hu2022lora}.
        We experiment with different \artifact sizes and choose LoRA rank and prefix-tuning cache size to align on memory consumption. 
        We evaluate the performance of the \artifacts on questions from the target dataset (\longhealth or \qasper) using the same protocol as in \Cref{fig:tradeoff-within} and also on questions from MMLU~\cite{hendrycks2020measuring} that are unrelated to the corpora. 
        (\textbf{Left}) The $x$-axis shows accuracy on MMLU and the $y$-axis shows accuracy on the target dataset. Each point represents a different \artifact size.
        (\textbf{Center}) The $x$-axis shows \artifact size in GB, and the $y$-axis shows accuracy on MMLU.
        (\textbf{Right}) The $x$-axis shows self-study duration in training steps, and the $y$-axis shows accuracy on MMLU. The shade of the points represents the size of the \artifact.
        }
    \label{fig:parameterization}
    \vspace{-2mm}
\end{figure*}

\section{Extended Results}
\label{app:results}

In this section, we ablate the main design choices of \artifacts and \method.

\subsection{\artifact design choices: parameterization and initialization}
\label{app:results-parameterization}
In our experiments, we parameterize the \artifact with a simplified version of prefix-tuning and initialize with a truncated KV-cache (see \Cref{sec:artifact-parameterization}). In this section, we describe ablation experiments motivating these design choices.
First, we compare two different \artifact parameterizations (\Cref{fig:parameterization}): simplified prefix-tuning~\cite{li2021prefix} and low-rank adaptation (LoRA)~\cite{hu2022lora}.
Then, we demonstrate the importance of proper \artifact   initialization (\Cref{fig:intialization}).

\paragraph{Parameterization}
We evaluate \artifacts trained on corpora from \longhealth or \qasper on both \textit{in-domain} (\ie questions from \longhealth or \qasper) and \textit{out-of-domain} (\ie questions from an unrelated benchmark, MMLU~\cite{hendrycks2020measuring}) queries.

We find that the prefix-tuning parameterization is more effective than a memory-matched LoRA parameterization on both in-domain and out-of-domain queries. 
This is illustrated in \Cref{fig:parameterization} (Left), where we see that prefix-tuning occupies the top-right corner of the plot (high accuracy on both MMLU and the target dataset).

Notably, we find that as we increase the \artifact size with LoRA tuning, performance on out-of-domain queries (MMLU) drops significantly. At 1.06 GB (LoRA rank 1632), MMLU accuracy drops from $60.0\%$ to $45.3\%$. 
This drop in performance is highly correlated with the size of the \artifact, suggesting that LoRA is not well-suited to large Cartridges, which we show in \Cref{fig:tradeoff-within} are important for recovering ICL performance.
In contrast, with prefix-tuning the accuracy only drops to $54.3\%$ at 1.06 GB. This degradation is mostly invariant to the size of the \artifact ($54.7\%$ at 0.15 GB), demonstrating that out-of-domain performance is robust across \artifact sizes.

On in-domain queries, prefix-tuning also outperforms LoRA, but the gap is smaller. 
Across all \artifact sizes, the best \longhealth accuracy prefix-tuning achieves is $55.6\%$ at $0.96$ GB, while the best LoRA accuracy is $47.25\%$ at $0.26$ GB. 
Interestingly, LoRA accuracy at the largest \artifact sizes is lower; $41.3\%$ at $0.96$. 
It is possible that this is due to the out-of-domain degradation of LoRA we discussed above. 
Since queries in \longhealth test set are quite different from the synthetic queries generated by \method (\eg they are multiple choice and require some complicated reasoning traces), out-of-domain robustness may be also important for ``in-domain'' performance. 

It isn't clear why prefix-tuning is so much more robust than LoRA to out-of-domain performance degradation. 
It is surprising given the similarity between a KV-cache and an MLP -- both are linear transformations separated by a non-linearity. 
It is possible that this is due to the difference in the activation function (SiLU vs. Softmax). 
We leave a more detailed investigation into the root cause of this difference for future work.

\paragraph{Initialization}
\begin{figure*}[ht]  
    \centering
    \includegraphics[width=0.8\textwidth]{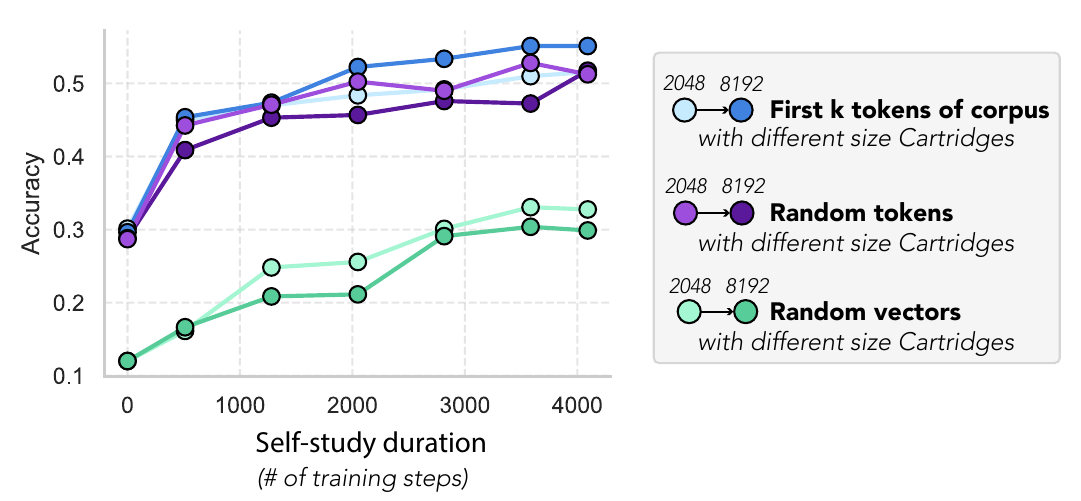}

    \caption{\textbf{Ablating \artifact initalization}. We train a \artifacts using \method on the corpora from \longhealth with 3 different initialization strategies. The $x$ axis is the number of training steps and the $y$ axis is the  accuracy on \longhealth. The blue lines are the results when initializing the \artifact using the KV cache from the first $k$ tokens of the document. The purple lines are initializing the \artifact from the KV cache of unrelated text. The green lines is initializing the \artifact with random vectors. Initializing from the first $k$ tokens leads to slightly stronger results than initializing from the KV cache of random text. This difference may be more prominent on other corpora where the first $k$ tokens are more relevant to solving the downstream task.}
    \label{fig:intialization}
\end{figure*}

The standard way of initializing a $k$ token \artifact in our main paper is using the KV cache from the first $k$ tokens of the source document. In \Cref{fig:intialization}, we ablate different initialization source. We try two additional initalizations: \textit{random vectors} and \textit{random tokens}.

For \textit{random vectors}, we simply initialize the parameters of the \artifact from a component-wise standard normal distribution. For \textit{random tokens}, we initialize the \artifact as the KV cache of the first $k$ tokens of arbitrary text (specifically, the \href{https://en.wikipedia.org/wiki/Gradient}{Wikipedia page for gradient}). The important difference between the these two strategies is that for \textit{random tokens} the initial \artifact is "valid" KV cache produced by the model, while for \textit{random vectors} it is not. 

\begin{figure*}[ht]  
    \centering
    \includegraphics[width=\textwidth]{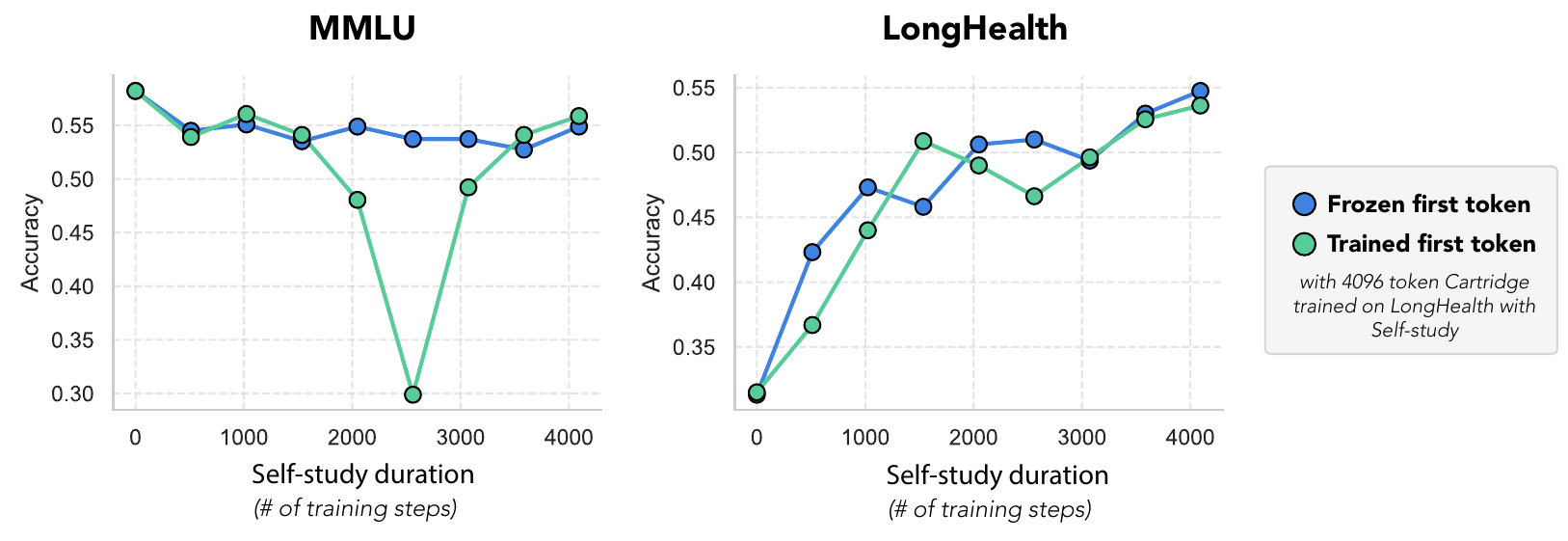}

    \caption{\textbf{Freezing the attention sink}. In both plots, the y-axis is accuracy and the x-axis is training step. The green line which corresponds to a run where we allow a trainable first token. (\textbf{Left}) The y-axis MMLU accuracy. This plot exemplifies the training instability we observed when the key and value vectors were trainable. The MMLU score dips to below 30\% before recovering. (\textbf{Left}) The y-axis is accuracy on questions from \longhealth.}
    \label{fig:freeze}
\end{figure*}
\paragraph{Freezing the attention sink}

A small yet important detail of training a \artifact is that we do not let the first token's key and value vectors to be trainable. As studied in \cite{xiao2024efficientstreaminglanguagemodels}, the first key vector, which corresponds to the beginning of sequence token and is thus the same for \textit{every sequence}, acts as an "attention sink".  We observed that when training a \artifact, allowing those key and value vectors to be trainable led to training instability (see \Cref{fig:freeze}). For example, on some runs the MMLU accuracy would dip to below $30\%$.

\subsection{\method design choices: data-generation and objective}
In \method training we use a seeded data-generation process and a context-distillation training objective (see \Cref{sec:method}). In this section, we ablate these design choices, comparing against the performance of \method with simpler data-generation and objectives.

\paragraph{Data Generation} In \Cref{sec:method-data}, we describe how we use five different seed prompt types when generating data with \Cref{alg:synthetic-generation}. These prompt types, \textit{structuring}, \textit{summarization}, \textit{question}, \textit{use cases}, and \textit{creative}, are described in more detail in \Cref{app:method-data-seed}.

In this section, we compare the performance of \method with these five prompt types against \method with a single prompt: \textit{``Please generate a single chat message to begin a conversation about the information in the corpus. Ask a question about the corpus or make a request."}

Across three datasets, we find that using the five different prompt types during \method leads to higher quality \artifacts (see \Cref{fig:logits}). 
On \mtob with \artifacts of size 1024 tokens, we see a $7.9$ point ChRF improvement ($24.1 \rightarrow 32.0$).
On \longhealth, the improvement is $5.5$ accuracy points ($45.8 \rightarrow 51.3$). 

Interestingly, on \qasper, we see no benefit from using the five different prompt types. 
It is possible this is because the queries in the \qasper dataset are mostly factual questions that do not require complex reasoning like \longhealth and \mtob do.

\begin{figure*}[h]  
    \centering
    \includegraphics[width=\textwidth]{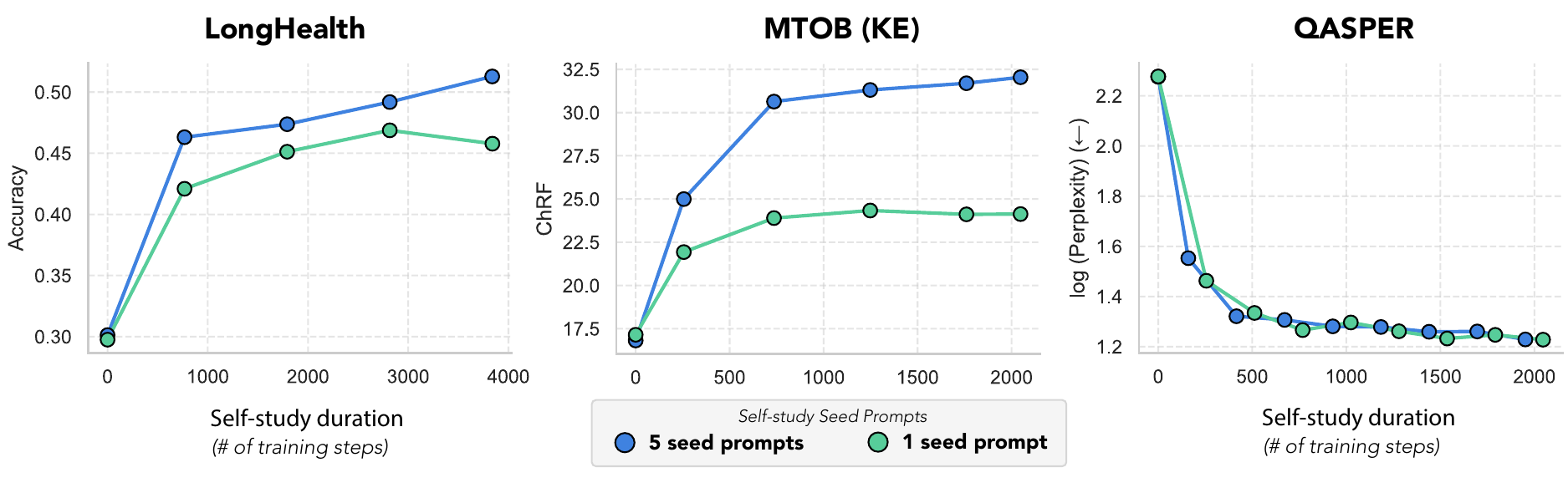}

    \caption{\textbf{Diverse seed prompts improve quality.}
    We generate synthetic data according to \Cref{alg:synthetic-generation} and ablate the choice of seed prompts sampled on Line 2.
    We consider two approaches: using a single, broad seed prompt (Green) or randomly sampling one of five different types of seed prompts (Blue).
    We train \artifacts using self-study with these two strategies on \longhealth, \mtob and \qasper corpora.
    In all plots, the $x$ axis is the number of training steps, and the $y$ axis is either accuracy (for \longhealth and \mtob) or perplexity on ground truth answer (for \qasper).
    We use an \artifact size of 1024 tokens.
         }
    \label{fig:seeds}
\end{figure*}

\paragraph{Training Objective} 
\label{app:obj}
In \Cref{sec:method}, we describe the context-distillation objective we use~\cite{snell2022learning,kim2016sequence,bhargava2024prompt}. 
This approach requires that we collect top output probabilities from the in-context model's output distribution during data generation.
A simpler alternative would be to just use a next-token prediction objective with a cross-entropy loss.

In our comparison, we find that this simpler objective underperforms the context-distillation objective (see \Cref{fig:logits}).
Most notably, on \mtob with 2048 token \artifacts, context-distillation outperforms next-token prediction by $8.3$ ChRF points ($24.9 \rightarrow 33.2$). 
On LongHealth, the gap is $3.7$ accuracy points ($47.6 \rightarrow 51.3$).
\begin{figure*}[ht]  
    \centering
    \includegraphics[width=\textwidth]{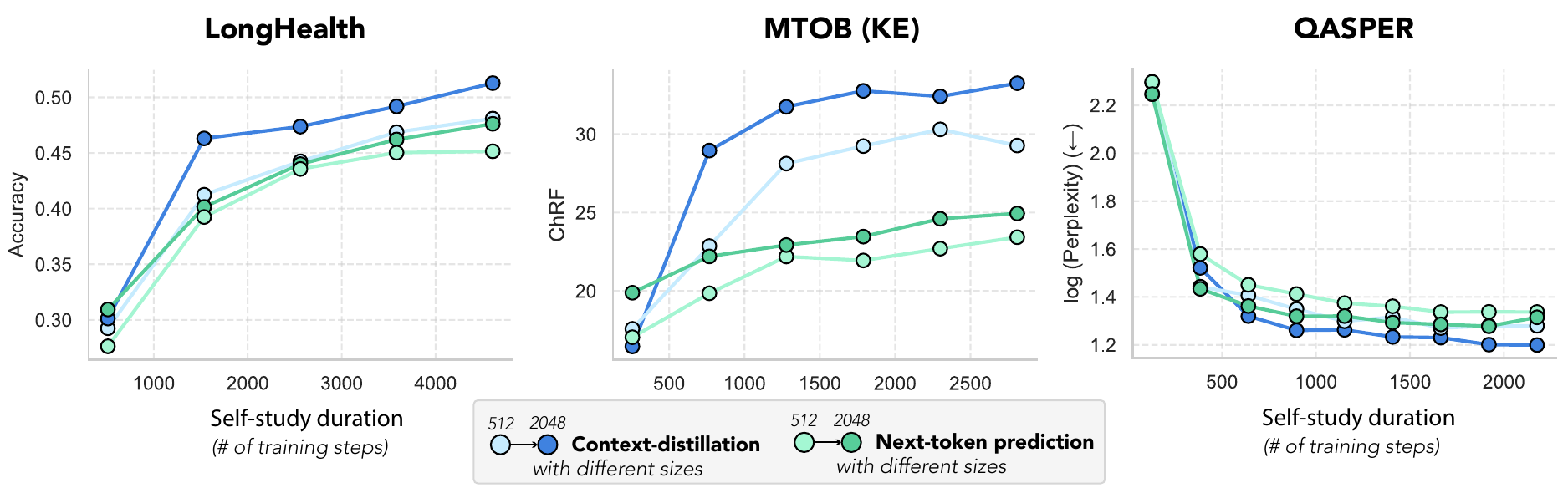}

    \caption{\textbf{Context-distillation objective improves training efficiency}. We train \artifacts using \method on the corpora from \longhealth (Left), \mtob (Center) and \qasper (Right) using two loss functions: a next token prediction loss (green) and a distillation loss (blue). 
        We evaluate the performance of the \artifacts on questions from the target dataset (\longhealth, \mtob or \qasper) using the same protocol as in \Cref{fig:scaling-indexing}. In all plots, the $x$ axis is the number of training steps, and the $y$ axis is either accuracy (for \longhealth and \mtob) or perplexity on ground truth answer (for \qasper).
         The shade of the points represents the size of the \artifact. Using a distillation loss achieves higher accuracy (or lower perplexity for \qasper) across datasets and \artifact sizes.}
    \label{fig:logits}
\end{figure*}

As shown in \Cref{fig:logits}, quality seems to be consistently improving with more \method compute.
It is possible, therefore, that by spending more during \method with the next-token prediction objective, we could close the gap. 
However, for a fixed amount of \method compute, context-distillation is considerably more effective.

These results demonstrate how context-distillation plays an important role in efficiently recovering ICL performance with \method.

\subsection{Throughput measurement details}

We provide details for the throughput measurements in Figure~\ref{fig:micros}. We use the state-of-the-art SGLang inference system, with default parameters \cite{zheng2024sglang}. We measure throughput on a single H100 GPU.

We first determine the largest batch size $b$ that fits in GPU memory, given a cache of size $k$ tokens. We then randomly initialize $b$ \artifacts of size $k$ and pre-load the \artifacts into GPU memory. We finally measure the time taken to decode $128$ tokens per sequence. The \artifacts and decoded tokens are appended to a KV-cache during generation. We report the average of $5$ iterations after using $3$ warm-up iterations.

\section{Extended Related Work}
\label{app:related-work}
In this section, we provide a more in-depth discussion of the place our work occupies in the broader literature.
The structure below mirrors the structure of our paper: first we discuss work related to the parameterization and initialization of \artifacts (\Cref{app:related-work-artifact}), then we cover work that inspired the design of \method (\Cref{app:related-work-method}), and finally we describe other approaches aimed at reducing the size of the KV-cache, many of which we compare against in our experiments (\Cref{app:related-work-reducekv}).

\subsection{Prior work related to the parameterization of \artifacts}
\label{app:related-work-artifact}

Below we discuss prior work from the parameter-efficient fine-tuning literature that inform the way we parameterize \artifacts in our work. 

\subsubsection{Parameter-efficient Fine-tuning (PEFT)}
\label{app:related-work-peft}
In order to adapt large language models (LLMs) to particular domains or tasks in a more compute and memory-efficient manner, several parameter-efficient fine-tuning (PEFT) methods have been developed. Some of the most widely used PEFT methods include Low-Rank Adaptation (LoRA) \cite{hu2022lora}, prefix-tuning \cite{li2021prefix}, and prompt-tuning \cite{lester2021power}. 

Leveraging prior observations that fine-tuned language models exhibit an intrinsic low rank structure, Hu \etal propose LoRA, which freezes model parameters and injects trainable rank decomposition matrices between each transformer layer. LoRA exhibits on-par or better fine-tuning quality while reducing the number of trainable parameters by 10,000 times and the GPU memory requirement by 3 times \cite{hu2022lora}. 

Li \etal and Lester \etal both take a different approach to lightweight fine-tuning, proposing tunable "prefixes" and "soft prompts" respectively to prepend to queries in order to steer the model to desired outputs. Li \etal proposes prefix-tuning, which learns a continuous representation for the activation of the prefix at each transformer layer. These learned activations are then prepended to activations obtained by passing the input prompt through the frozen transformer. In contrast, Lester \etal proposes prompt-tuning, which optimizes at the discrete token level and prepends a series of learnable tokens to the input prompt. Both methods show strong performance while greatly reducing the number of learnable parameters and improving compute and memory efficiency for language model adaptation. 

Principal Singular values and Singular vectors Adaptation (PiSSA) \cite{meng2024pissa} is another more recent PEFT method that attempts to ameliorate the slow convergence problems of LoRA. PiSSA initializes the LoRA rank decomposition matrices with the principal components of the original matrix, and exhibits faster convergence and enhanced performance compared to LoRA on several tasks, including GSM8K and MATH. 

Several of these methods, especially LoRA, have been adapted specifically for distilling knowledge provided in context into the parameters of a language model. Some of those methods are described in the sections below, and this work is an extension of prefix-tuning for long-context tasks.

\subsubsection{Parameter-efficient Adapter Composition and Merging}

A number of works have explored the idea of composing multiple different parameter-efficient adapters (\eg LoRAs) by summing them together, concatenating them, or using a dynamic mixture of experts~\cite{zhao2024merging,huang2023lorahub,xiao2024configurable,zhao2024loraretriever,yadav2024survey,wu2024mixture,gou2023mixture,li2024mixlora}.
For example, Huang \etal propose LoraHub, a framework for dynamically weighting and composing multiple language model adapters~\cite{huang2023lorahub}. Given a set of LoRA modules for different upstream tasks and new unseen task with in-context examples, LoraHub dynamically weights the LoRAs and composes a new LoRA module for the task.   
Similarly, Zhao \etal propose a method for dynamically \textit{retrieving} the most relevant language model LoRAs for a given task~\cite{zhao2024loraretriever}.

\subsubsection{Parametric Knowledge Injection}

Several recent works have explored methods for integrating external knowledge directly into model parameters, known as parametric knowledge injection ~\cite{kujanpaa2024knowledge,mao2025lift,su2025parametricrag,caccia2025training,kuratov2025cramming}. 
To the best of our knowledge, these studies are the closest in scope to ours.
Like ours, these works address the problem of parametric knowledge injection: how to store large text corpora within parameters of a language model.
Some use simple synthetic data generation pipelines or context-distillation objectives.
Unlike our work, these studies do not highlight the memory reduction and throughput advantages of parametric knowledge injection techniques. 
We highlight other differences below.

One parametric knowledge injection method, recently proposed by Kujanpaa \etal, is prompt distillation, in which a teacher model with access to privileged knowledge generates question-answer pairs. These pairs are then used to train a LoRA adapter for a student model (identical to the teacher model, but without access to privileged information) using a distillation objective (i.e. mimicking the teacher's full token distribution) ~\cite{kujanpaa2024knowledge}.
This closely resembles our context-distillation objective, which we also found works better than next-token prediction.
However, unlike our work, Kujanpaa \etal only train LoRA adapters of a single size (rank 1024) and don't assess memory reductions with respect to full in-context learning. 
Indeed, they do not evaluate against long-context ICL baselines at all, focusing instead on a comparison with RAG. 
Furthermore, they evaluate on a relatively simple long-context setting -- a concatenation of SQUAD passages~\cite{rajpurkar2016squad} -- which does not exhibit long range dependencies or require reasoning the way \mtob and \longhealth do.

Similarly, Mao \etal propose Long Input Fine-tuning (LIFT), which fine-tunes a language model using a typical next-token prediction objective on overlapping segments of the corpus, as well as instruction tuning on question answer pairs generated from the corpus. 
Unlike our work, Mao \etal find that synthetic Q/A pairs ``offer minimal benefit and can even degrade performance due to overfitting"~\cite{mao2025lift}. 
The difference in our findings is perhaps due to the fact that they only generate \textit{ten} synthetic examples, whereas we generate \textit{tens of thousands}. 
Furthermore, they use a weaker ICL baseline (Llama 3 8B) that only has 8k tokens of context. Any contexts longer than 8k tokens are truncated before being fed to the ICL baseline. 

Concurrent work on \textit{deep context distillation} performs knowledge injection with synthetic data and a context distillation objective~\cite{caccia2025training}. 
In this work, the authors only report performance with LoRA adapters and do not explore a prefix-tuning parameterization. 
In further contrast to our work, their focus is not on memory reductions or throughput improvements. They only report performance with a single adapter size (rank 16 LoRA adapters), and they do not report throughput improvements. Instead, the paper highlights the ``plug-and-play" nature of the method. 

Finally, Su \etal proposes Parametric Retrieval Augmented Generation (Parametric RAG), in which each document has a corresponding LoRA adapter, trained on an augmented dataset consisting of the document, rewritten versions of the document, and question-answer pairs generated from the document. At inference time, a retriever is used to determine relevants documents, and the corresponding LoRA adapters are merged \cite{su2025parametricrag}. This method demonstrates significant gains over RAG on a variety of tasks, including WikiMultihopQA.

\subsection{Prior work related to \method}
\label{app:related-work-method}

\subsubsection{Self Distillation and Context Distillation}
Self-distillation is another method used to internalize the performance gains provided by information in context (e.g. scratchpads, informative instructions) into the model parameters. In "Learning by Distilling Context", the authors distill a model with instructions and scratchpads in context into parameters  by conditioning the model on “[instructions] + [task-input]” to predict “[scratch-pad] + [final answer]”; then fine-tuning the same model to predict its own “[final answer]” conditioned on the “[task-input]”, without seeing the “[instructions]” or using the “[scratch-pad]” ~\cite{snell2024scaling}.

\subsubsection{Synthetic Data Generation}
Due to the ubiquitous need for high quality data for fine-tuning (e.g. for use with the methods described above), a large body of work has focused on generating high quality synthetic data ~\cite{nayak2024learning} \cite{abdin2024phi} \cite{gandhi2024datatune} \cite{riaz2025metasynth}. For example, Bonito is a model that is fine-tuned to generate synthetic data \cite{nayak2024learning}, and MetaSynth is a method proposed by Riaz \etal that uses a language model to orchestrate several expert LLMs for domain-specific synthetic data generation \cite{riaz2025metasynth}. The training process for Phi-4, a 14 billion parameter language model, also incorporates significant amounts of synthetically generated data \cite{abdin2024phi}. Incorporating synthetic data, in conjunction with new post-training techniques, allows Phi-4 to surpass its teacher model on STEM QA tasks, as well as perform well for its size on reasoning benchmarks. These works demonstrate the potential for synthetic data generation methods to augment the capabilities of language models. 

\subsection{Reducing the size of the KV cache}
\label{app:related-work-reducekv}
In this section, we discuss existing approaches for reducing the size of the KV cache.

First, in \Cref{app:related-work-reduce-arch}, we describe works that propose architectural changes to the multi-head attention operation, which reduce the memory footprint of the KV cache. Next, in \Cref{app:related-work-reducekv-prompt}, we discuss \textit{prompt compression} methods, which reduce the size of the KV cache by converting a long sequence of input embeddings into a shorter one. They can be split into hard-token methods, which output discrete tokens from the vocabulary, and soft-token methods, which output new token embeddings not from the vocabulary.
Finally, in \Cref{app:related-work-reducekv-kvcache}, we describe \textit{KV cache compression} methods. These methods directly modify the key and value matrices in the KV cache. Compared with prompt compression methods, these are more expressive because they can produce a KV cache that no sequence of input embeddings could have produced. 

The methodology proposed in our work relies on cache-tuning, which could be viewed as a form of KV cache compression.

\subsubsection{Prompt compression}
\label{app:related-work-reducekv-prompt}
\paragraph{Hard-token prompt compression}
Some works aim to reduce the size of KV cache by converting a longer text into a shorter text~\cite{jiang2023llmlingua,li2023unlocking,chuang2024learning,zhang2024adacomp,pan2024llmlingua}. These methods are typically referred to as \textit{hard-token} prompt compression methods because the resulting KV cache comes from discrete tokens from the vocabulary. Compared with soft-token prompt methods, these methods work well with black-box API models. 

These methods can be broadly classified into two categories: filtering and summarization based methods. 
Filtering methods cut text from the original prompt using heuristics such as self-information. For example, LLMLingua and Selective-Context use a smaller LLM to filter a long prompt (\textit{e.g.} dropping redundant tokens) before passing it to the main model~\cite{jiang2023llmlingua,li2023unlocking}.
Summarization methods paraphrase a long prompt into a smaller number of tokens~\cite{chuang2024learning}.

\paragraph{Soft-token prompt compression with adapted LLMs}
In one line of work, researchers train a model (typically an adapted LLM) to compress a long prompt into a smaller number of soft tokens~\cite{chevalier2023adapting,yen2024long,ge2023context,mu2023learning,qin2023dodo}.

For example, \textit{Autocompressors} and \textit{In-context Autoencoders} (ICAE) are LLMs that are fine-tuned to output embeddings which can be used in soft-token prompts~\cite{chevalier2023adapting,ge2023context}. Autocompressors are trained with full-parameter fine-tuning and leverage a recursive strategy to generate the soft prompts, whereas ICAEs are trained with LoRA and use a single forward pass to generate the soft prompts. 
A recent method, LLoCO, train domain-specific LoRA adapters that enable the decoder better leverage AutoCompressor embeddings~\cite{tan2024lloco}. This differs from \artifacts in that the LLoCO LoRA adapters are trained for a domain (\eg academic papers, news), not a specific document.
A number of other works also propose using an auxiliary model to produce soft-tokens from a long prompt~\cite{ge2023context,qin2023dodo}.
\textit{Gisting} is another method that differs from those above in that it uses the same LLM to compress the prompt into soft tokens as it uses to generate the response~\cite{mu2023learning}.

\paragraph{Soft-token prompt compression via gradient-descent}
Soft tokens can also be produced by optimizing input token embeddings with gradient descent. 
This idea, called \textit{prompt tuning}, was first proposed for the purpose of conditioning a frozen langauge model to perform specific tasks~\cite{lester2021power}. 
As such, it is an important part of the parameter-efficient fine-tuning literature and is discussed in more detail in \Cref{app:related-work-peft}. 
Since then, Li \etal has extended prefix tuning techniques to long-context settings, proposing a new method called prefix propagation, which conditions prefixes on previous hidden states to achieve superior performance on long-document tasks compared to prefix tuning~\cite{li2024mixlora}. 

\subsubsection{KV cache compression} 
\label{app:related-work-reducekv-kvcache}

\paragraph{Hard-token KV cache compression} Motivated by the observation that, in some settings, a small number of keys dominate the attention scores of subsequent queries, several works have proposed \textit{KV cache eviction policies} wherein keys and values are dynamically dropped during generation~\cite{ge2023model,zhang2023h2o,tang2024quest,oren2024transformers}. For example, H20 drops keys and values from \textit{generated tokens} based on a running sum of historical attention scores~\cite{zhang2023h2o}. Similarly, SnapKV drops keys and values from \textit{prompt tokens} based on a window of queries from the end of the prompt~\cite{li2024snapkv}.

A major limitation of eviction methods is that once a key is evicted, it cannot be recovered. 
Instead of evicting keys permanently, another line of work focuses on selectively loading keys from KV cache to SMs. 
While these works do not reduce memory consumption of the KV cache, they can speed up inference by making better use of GPU memory bandwidth~\cite{ribar2023sparq,tang2024quest}. For example, the Quest method estimates critical tokens at each decoding step and selectively loads them to SMs~\cite{tang2024quest}.

Compared with the hard-token \textit{prompt compression} methods, KV-cache compression methods allow fine-grained control at the level of an attention head. This means that a token can be dropped from one attention head but not another.

\paragraph{Soft-token KV cache compression with merging} In another line of work, instead of evicting tokens from the KV cache, researchers propose merging similar tokens ~\cite{wang2024model,zhang2024cam,wan2024d2o, liu2024minicache}. For example, Cache Merge (CaM) takes keys marked for eviction and merges them instead, using a weighting scheme based on attention weights \cite{zhang2024cam}. Wang \etal builds on this work by clustering key states into "merge sets" based on cosine similarity, and merging states within a "merge set" with a Gaussian kernel weighting scheme, which upweights states more similar to a pivotal state chosen as the token with the largest total attention score \cite{wang2024model}. Wan \etal expands on both these works with Dynamic Discriminative Operations (D2O), which performs optimizations at both the layer and token levels. D2O adjusts the KV cache budget for each layer based on its attention density and uses an exponential moving average mechanism to dynamically determine when a previously discarded token is similar enough to retained tokens to be merged back in \cite{wan2024d2o}. All of these works demonstrate promising results, offering similar or better performance on several tasks compared to a full cache with a 50\% or more reduction in cache size. However, there is still room for further improvement, as these methods still fail to match full cache performance in several tasks, and even a 50\% reduction in cache size may still be prohibitively expensive for very large models or very long contexts. Additionally, these works do not evaluate the effectiveness of these methods in long-context settings. 

\paragraph{Soft-token KV cache compression with low-rank projection} A number of works leverage the observation that the KV cache exhibits low-rank structure to develop compression methods ~\cite{yu2024effectively,chang2024palu, zhang2024lorc, zhou2025elitekv, saxena2024eigenattn}. Similar to compression methods based on merging, compression methods based on low-rank adaptation achieve performances similar to or exceeding full caches on several tasks at 50\% compression, while experiencing performance degradation upon further compression.

\paragraph{Soft-token KV cache compression with adapted LLMs}
Above we discussed how some works adapt an LLM to output a shorter sequence of soft tokens given a long context. Similarly, one could adapt an LLM to output a smaller KV cache given a long context. While less explored than the analagous prompt compression approach, there is at least one published method that falls into this category. 
In \textit{KV-distill}, the authors add LoRA adapters to an LLM's query projections and train them to to produce queries which aggregate information from prior tokens~\cite{chari2025kv}. 
The adapter is applied selectively to some tokens and only these tokens are kept in the KV cache.
The idea is that these selected tokens can act as sinks to collect information from prior tokens. 
The adapter is trained with a distillation objective between a compressed and uncompressed KV cache. However, unlike our work, KV-distill does not use any training at test time.

\paragraph{Soft-token KV cache compression with gradient-descent} The idea of treating the keys and value matrices in a KV cache as weights and training them with gradient descent was first discussed in the prefix-tuning paper~\cite{li2021prefix}.
In this work, the method was not applied to long-contexts, but rather as a parameter-efficient fine-tuning method that can be applied to training datasets with input-output pairs, so we discuss it in more detail in \ref{app:related-work-peft}. 
Since then, we are not aware of works that have applied this technique to handle long-contexts. 

\subsubsection{Architectural changes}
\label{app:related-work-reduce-arch}
A number of works have proposed architectural changes to the original multi-head attention (MHA) operation~\cite{vaswani2017attention} that reduce the memory footprint of the KV cache. 
Because they fundamentally alter the architecture, these methods are not immediately compatible with pre-trained models using the standard MHA operation. 

The earliest works in this direction developed fixed sparsity patterns in the attention map~\cite{beltagy2020longformer,child2019generating,zaheer2020big}. For example, many works use a sliding window sparsity pattern wherein each token attends to a fixed window of tokens around it. These approaches reduce the size of the KV cache because they require only keeping around a fixed number of tokens in the KV cache. 
More recently, some large language models have adopted sliding window sparsity in a subset of layers/heads~\cite{team2024gemma}.

While the methods above reduce the size of the cache by introducing sparsity at the token-level, another class of methods changes the structure of the attention heads. Multi-query attention (MQA), the earliest of such modifications, uses multiple query heads but only a single key and value head~\cite{shazeer2019fast}.
While MQA dramatically reduces the size of the KV cache, it can lead to a significant drop in the expressive power of the model. Grouped-query attention (GQA) is a middle ground between MQA and MHA that allows a group of query heads to attend to a single key and value head~\cite{ainslie2023gqa}. Many frontier models use GQA, including the Llama 3 architecture, which we use in our experiments~\cite{dubey2024llama3,jiang2024identifying,yang2024qwen2}.
More recently, a number of other architectural modifications have been proposed including including Multi-head Latent Attention~\cite{liu2024deepseek} and Tensor Product Attention~\cite{zhang2025tensor}.

In another line of work, researchers observe that without the softmax operation in the attention mechanism (\textit{i.e.} linearizing the attention operator), the KV cache can be faithfully represented by the fixed size matrix $K^\top V$~\cite{arora2024simple}. This allows us to represent the KV cache with a single matrix whose size is independent of the context length.

Indeed, a large body of work has focused on developing architectures with fixed-size memory consumption (\textit{i.e.} models that do away with the KV cache). Notable examples include state-space models~\cite{gu2023mamba}, RNNs~\cite{beck2024xlstm}, and other linear attention variants~\cite{arora2024simple,yang2024parallelizing}.

Prior work shows that there are tradeoffs between the memory consumption of an architecture and the ability of a model to perform recall-intensive tasks, when controlling for compute (\ie FLOPs)~\cite{arora2024simple}. 
In this context, our work shows that by increasing compute (\ie FLOPs), we can reduce the memory consumption of a model without sacrificing performance. 
In \Cref{app:theory}, we provide a prelinary theoretical analysis relating \method with recurrent architectures. However, future work should explore the relationship between \artifacts and recurrent models in more depth.

Most related to our work are recent architectures (\eg Titans~\cite{behrouz2024titans}, TTT~\cite{sun2024learning}) that use a constant-sized memory object (like in linear attention) but apply gradient descent-like memory updates~\cite{sun2024learning,yang2025parallelizinglineartransformersdelta,behrouz2025atlas,behrouz2024titans,behrouz2025s}.
Like our work, these architectures are motivated by the observation that gradient descent is very effective at compressing text into constant space and demonstrate the promise of using gradient descent at test time for long-context tasks. 
In contrast with our work, these architectures need to be trained from scratch, they have not been validated on large scale models, and do not match the quality of attention on recall-intensive tasks~\cite{arora2024simple,behrouz2025atlas}.

\subsubsection{Orchestration for long-context}
In this section, we describe strategies for managing long-contexts by orchestrating calls to LLMs.
For instance, the approach by \citep{russak2024writing} involves summarizing chunks of the context and then combining the summaries. Similarly, PRISM \citep{jayalath2024long} treats the context as a sequence of chunks, capturing key information in a structured data format. MemGPT \citep{packer2023memgpt} introduces a virtual memory paging system, drawing inspiration from operating systems. As context length reaches the limit of available memory, the system strategically determines which information to retain.

\subsubsection{Synthetic data generation}
A large body of work has focused on generating synthetic training data ~\cite{nayak2024learning, abdin2024phi,gandhi2024datatune,riaz2025metasynth}. For example, Bonito is a model that is fine-tuned to generate synthetic data \cite{nayak2024learning}, and MetaSynth is a method proposed by Riaz \etal that uses a language model to orchestrate several expert LLMs for domain-specific synthetic data generation \cite{riaz2025metasynth}. The training process for Phi-4, a 14 billion parameter language model, also incorporates significant amounts of synthetically generated data \cite{abdin2024phi}.

\section{Extended method description}
\label{app:method}

\label{app:method-data}
In this section, we detail the seed prompts and chunking strategy we used to train \artifacts with \method.

\subsection{\method seed prompts}
\label{app:method-data-seed}

As discussed in \Cref{alg:synthetic-generation}, we seed the synthetic conversation generation with a prompt that elicits conversations about different aspects of the document. 
For each conversation, we randomly sample one of the following functions and create a seed prompt by calling it:

\begin{exampleboxcode}[Structuring Seed Prompt Generator]
    \tiny
\begin{lstlisting}[language=Python]
def structuring_seed_prompt(**kwargs):
    DATA_FORMATS = [
        "JSON",
        "YAML",
        "TOML",
        "INI",
        "XML",
        "plain text",
    ]

    data_format = random.choice(DATA_FORMATS)

    EXAMPLES = [
        (
            "Can you structure the information in {{subsection}} of {{document}} related to {{something specific}} "
            f"in the following format: {data_format}? "
            "Be sure to include precise information like any dates, times, names, and numerical values.'"
        ...

    ]

    example = random.choice(EXAMPLES)

    return (
        f"Please generate a single chat message instructing an LLM to structure the information in {data_format}. "
        "Output only the chat message itself and absolutely nothing else. "
        "Make sure it is clear what section and document you are asking about. "
        f"The message can follow the following template, filling in details from the corpus: \n\n'{example}'"
    )

    \end{lstlisting}
\end{exampleboxcode}

\begin{exampleboxcode}[Summarization Seed Prompt Generator]
\begin{lstlisting}[language=Python]
def summarization_seed_prompt(**kwargs):
    prompts = [
        (
            "Please generate a single chat message instructing an LLM to summarize part of the corpus. "
            "Make sure the instruction is very explicit about the section of the corpus that you want to summarize. "
            "Include details (ids, names, titles, dates, etc.) that make it clear what you are asking about. "
        ),
        (
            "Please generate a single chat message instructing an LLM to summarize a section. "
            "Make sure the instruction is explicit about the section that should be summarized and the document it is from."
        ),
    ]
    prompt = random.choice(prompts)
    return prompt

    \end{lstlisting}
\end{exampleboxcode}

\begin{exampleboxcode}[Question Seed Prompt Generator]
\begin{lstlisting}[language=Python]
def question_seed_prompt(**kwargs):
    prompts = [
        (
            "Generate a question for an LLM that will test its knowledge of the information in the corpus above. "
            "In your question be sure to include details (ids, names, titles, dates, etc.) that make it clear what you are asking about. "
            "Output only a single question. Do NOT include any other text or explanation other than the question."
        ),
        (
            "Generate a message for an LLM that will test its knowledge of the information in the corpus above."
            "Be sure to include details (ids, names, titles, dates, etc.) in the question so that it can be answered without access to the corpus (i.e. closed-book setting). "
            "Output only a single question. Do NOT include any other text or explanation other than the question."
        ),
        (
            "You are helping to quiz a user about the information in the corpus. "
            "Please generate a question about the subsection of the corpus above. "
            "Be sure to include details (ids, names, titles, dates, etc.) in the question to make it clear what you are asking about. "
            "Answer only with the question, do not include any other text."
        ),
    ]
    prompt = random.choice(prompts)
    return prompt
    \end{lstlisting}
\end{exampleboxcode}

\begin{exampleboxcode}[Use Case Seed Prompt Generator]
\begin{lstlisting}[language=Python]
def use_case_seed_prompt(**kwargs):
    prompt = (
        "You are working to train a language model on the information in the following corpus. "
        "Your primary goal is to think about practical, real-world tasks or applications that someone could achieve using the knowledge contained within this corpus. "
        "Consider how a user might want to apply this information, not just recall it. "
        "After considering potential use cases, your task will be to generate a sample question that reflects one of these downstream applications. "
        "This question/instruction/task should be something a user, who has access to this corpus, might ask when trying to accomplish their specific goal. "
        "Output only a single question. Do NOT include any other text or explanation other than the question."
    )
    return prompt

    \end{lstlisting}
\end{exampleboxcode}

\begin{exampleboxcode}[Creative Seed Prompt Generator]
\begin{lstlisting}[language=Python]
def creative_seed_prompt(**kwargs):
    prompt = [
        (
            "You are having a creative conversation inspired by the information in the corpus. "
            "Please generate a question for your conversation partner to start off the discussion. "
            "Answer only with the question, do not include any other text."
        ),
    ]
    return random.choice(prompt)
\end{lstlisting}
\end{exampleboxcode}

\subsection{\method chunking}

For the \method data generation process, we extract uniformly random token-level chunks from the input corpus $\mathcal{C}$. A corresponding textual description is generally prepended to each chunk $\tilde{c}$ to contextualize it when generating the seed prompt. This approach helps the model focus on different parts of the corpus and generate diverse synthetic examples. The specific chunking parameters and descriptions are tailored to each dataset:

\begin{itemize}[leftmargin=*]
    \item \textbf{\longhealth:} Chunks are sampled with a minimum size of 512 tokens and a maximum size of 4096 tokens. The accompanying description is: \textit{`Below is a section of a patient's medical record. It is part of a larger corpus of medical records for $N_\text{patients}$ different patients.'}
    \item \textbf{AMD/FinanceBench:} Fixed-size chunks of 8192 tokens are utilized. No specific descriptive text is prepended to these chunks.
    \item \textbf{\mtob:} Chunks are sampled with a minimum size of 512 tokens and a maximum size of 4096 tokens. The description used is: \textit{`The following is an excerpt from a grammar book about the Kalamang language.'}
    \item \textbf{\qasper:} Following our general methodology, chunks are sampled with a minimum size of 512 tokens and a maximum size of 4096 tokens. A generic description is used to contextualize the chunk as an excerpt from a research paper, in line with the nature of the Qasper dataset.
\end{itemize}

\label{app:method-data-chunk}

\section{Datasets}
\label{app:datasets}

\subsection{ \genconvo }
To evaluate the ability of our approach to handle diverse queries over long documents, we generated the \genconvo dataset. We created \genconvo using the AMD 2022 10-K filing, a document from the FinanceBench corpus \cite{islam2023financebench}. The primary purpose of \genconvo is to simulate a wide range of tasks a user might ask a model to perform given a long document, thereby testing the model's comprehension, reasoning, and ability to extract varied types of information. The generation process relies on Claude Sonnet 3.7 \cite{anthropic2024claude} and is structured as follows:

\begin{enumerate}[leftmargin=*]
    \item \textbf{Document Input:} The entire source document (e.g., the AMD 2022 10-K, which is less than 200,000 tokens and fits within the model's context window) is provided to Claude Sonnet 3.7.
    \item \textbf{Question Generation:} A series of distinct prompt templates (detailed below), designed to elicit different reasoning traces (e.g., factual recall, synthesis, multi-hop reasoning), are used to generate questions. For the given document and each prompt template, we ask the model to generate 16 unique questions. This involves providing the model with the full document content alongside the specific question-generation prompt.
    \item \textbf{Answer Generation:} Subsequently, for each generated question, Claude Sonnet 3.7 is prompted again with the original full document and the generated question to produce an answer. This process ensures that the answers are grounded in the provided document.
\end{enumerate}

We hope \genconvo provides a challenging benchmark that moves beyond simple fact retrieval, assessing a model's capacity for deeper understanding and more complex information processing over long contexts. The following prompt templates were utilized for the question generation phase:

\begin{examplebox}[Factual Prompt Template]
    \small
    \ttfamily
Please generate a question to test someone's
ability to remember factual details from the document. The answer should be a few
tokens long and be a factual detail from the statement, such as a number, entity,
date, title, or name.

This question should not be common knowledge: instead, it should be something
that is only answerable via information in the document.

\end{examplebox}

\begin{examplebox}[Knowledge Prompt Template]
    \small
    \ttfamily
Please generate a question that requires
combining information mentioned both inside and outside the document. 

This question should require using a fact from the document and also a fact that
you are confident about, but is not mentioned in the document. For instance: 
- What are the founding dates of the companies that got acquired this year?
  This is a good question because the names of the acquired companies are
  mentioned in the document and the founding dates are not mentioned.
- What is the name of the CEO's spouse?  This is a good question because the
  name of the CEO is mentioned in the document and the spouse's name is not
  mentioned. 

The answer should be a fact that is a few tokens long such as a number, entity,
date, title, or name.
\end{examplebox}

\begin{examplebox}[Disjoint Prompt Template]
    \small
    \ttfamily
Please generate a multi-hop question that
tests someone's ability to use factual information mentioned in at least two
very different sub-sections of the document. 

This question shouldn't be a standard question about this kind of document.
Instead, it should ask about two particularly disconnected ideas, like
comparing information about the amount of owned space for the company
headquarters with the amount of dollars of estimated liability or comparing
the revenue number with the number of employees.

This question should also test one's ability to do retrieval: do not give
away part of the answer in the question. Ensure that for one to get the
correct answer to the question, they need to understand the document.

The answer should be a short: for example, a number, entity, date, title,
or name.
\end{examplebox}

\begin{examplebox}[Synthesize Prompt Template]
    \small
    \ttfamily
Please generate a question that requires
synthesizing and aggregating information in the document. 

For instance, you could ask someone to summarize a page of the document, list
all the key competitors mentioned in the document, or summarize the company's
business model.
\end{examplebox}

\begin{examplebox}[Structure Prompt Template]
    \small
    \ttfamily
Please generate a question that requires
understanding the structure of the document. 

This question should be more about the structure of the document, rather than
the precise statement details. For instance, you could ask someone to list the
titles of all the sections in the document, describe the document structure,
report the total number of pages, ask which section amongst two sections comes
first, or report the section with the largest number of tables.

\end{examplebox}

\begin{examplebox}[Creative Prompt Template]
    \small
    \ttfamily
Please generate a question about the
document to test someone's ability to comprehend the content of the document.
This question specifically should be focused on their ability to generalize the
information about the document to a strange question of sorts.

This question shouldn't be a standard question about this kind of document,
it should ask to do something abnormal and creative, like writing a poem
about a financial document.
\end{examplebox}

\begin{examplebox}[Counting Prompt Template]
    \small
    \ttfamily
Please generate a question that requires
counting how frequently different events occur in the document.

This question should be about statistical properties of the document, rather
than the statement details. For instance, you could ask someone to count the
number of times the word "million" is mentioned or count the length of the
shortest section title.

The answer should be a number.
\end{examplebox}

\begin{examplebox}[Reasoning Prompt Template]   
    \small
    \ttfamily
Please generate a question that requires
mathematical reasoning over the values in the document. 

This question should require going beyond the facts directly mentioned in the
statement, such as asking to compute the percentage increase in revenue between
two years, find the largest expense category, or calculate difference in profit
between two years. 

The answer should be a number.
\end{examplebox}

\subsection{\longhealth}
\longhealth is a benchmark for evaluating large language models ability to analyze and interpret long clinical texts~\cite{adams2024longhealth}. The benchmark consists of 20 fictional clinical case reports (each containing between 5,090 and 6,754 word) and 400 multiple-choice questions based on them. 

In our experiments, the context $\ctx$ consists of the reports for a \textit{panel} of $n$ patients. 
We use $n=10$ patients, with a full panel of approximately 100k tokens, which fits in the context length of the \llamathree models.

The questions are categorized into information extraction, negation, and sorting. 

A \textbf{sorting} question is included below:
\begin{examplebox}
    \small
    \ttfamily
    Please answer the question below about the following patient: ID patient\_03, Name: Mr. John Williams, Birthday: 1956-08-08 00:00:00, Diagnosis: Multiple Myeloma
    
    <question>\\
    Mr. Williams received multiple radiologic examinations. In which order did she receive them?\\
    </question>
    
    <options>\\
    CT Whole Body > MR Spine Scan > CT Spine Scan > PSMA-PET-CT Scan > CT Chest > CT Whole Body > Whole Body CT scan\\
    Whole Body CT scan > CT Spine Scan > CT Whole Body > MR Spine Scan > CT Chest > PSMA-PET-CT Scan > CT Whole Body.\\
    CT Whole Body > CT Whole Body > CT Chest > CT Chest > PSMA-PET-CT Scan > MR Spine Scan > CT Spine Scan > Whole Body CT scan > Chest X-ray\\
    CT Chest > CT Spine Scan > CT Whole Body > Whole Body CT scan > PSMA-PET-CT Scan > MR Spine Scan > CT Whole Body\\
    Whole Body CT scan > CT Spine Scan > CT Whole Body > MR Spine Scan > CT Chest  > CT Whole Body  > PSMA-PET-CT Scan\\
    </options>
    
    You should first think step by step. Then give your final answer exactly as it appears in the options. Your output should be in the following format:\\
    <thinking> \{\{YOUR\_THOUGHT\_PROCESS\}\} </thinking>\\
    
    <answer>\\
    \{YOUR\_ANSWER\}\\
    </answer>
\end{examplebox}

An example of a \textbf{negation} question is included below:
\begin{examplebox}
    \ttfamily

    Please answer the question below about the following patient: ID patient\_01, Name: Anna Sample, Birthday: 1970-01-01 00:00:00, Diagnosis: DLBCL
    
    <question>\\
    Which of these examinations were never performed in Mrs. Sample?\\
    </question>
    
    <options>\\
    Bone marrow aspiration\\
    CSF aspiration\\
    MRI of the head\\
    Pulmonary function testing\
    Cardiac stress testing\\
    </options>
    
    You should first think step by step. Then give your final answer exactly as it appears in the options. Your output should be in the following format:\\
    <thinking> \{\{YOUR\_THOUGHT\_PROCESS\}\} </thinking>\\
    
    <answer>\\
    \{YOUR\_ANSWER\}\\
    </answer>
\end{examplebox}

\subsection{\mtob}
The Machine Translation from One Book (MTOB) benchmark tests a large language model's ability to learn to translate between English and Kalamang, a low-resource language with virtually no web presence \cite{tanzer2023benchmark}. The core task is to perform translation (Kalamang to English, and English to Kalamang) by primarily relying on a single comprehensive grammar book and a small set of accompanying linguistic resources. In our work, we focus on translating from Kalamang to English.

The source documents provided by the MTOB benchmark are:
\begin{itemize}[leftmargin=*]
    \item \textbf{A grammar of Kalamang}: A comprehensive grammar textbook, with the original source provided in \LaTeX{} format. This book details the phonology, morphology, and syntax of Kalamang.
    \item \textbf{Bilingual Word List (W)}: A list of Kalamang words with their part-of-speech tags and English descriptions.
    \item \textbf{Parallel Kalamang-English Corpus (S)}: A collection of 375 paired Kalamang-English sentences.
\end{itemize}

The MTOB authors preprocessed the grammar textbook from its original \LaTeX{} source into several plaintext splits for their baseline experiments. These include:
\begin{itemize}[leftmargin=*]
    \item \textbf{$G^m$ (Medium-length chunk)}: A plaintext segment of approximately 50k tokens consisting of an overview chapter, a morpheme table from the grammar book, and the complete bilingual word list (W).
    \item \textbf{$G^l$ (Long-length chunk)}: A larger plaintext segment of approximately 100k tokens, containing chapters from the grammar book that the MTOB authors deemed most important for the translation task.
    \item \textbf{Full Plaintext Textbook (G)}: The entire grammar book converted to plaintext.
\end{itemize}

The combination of the long-length chunk ($G^l$), the parallel sentences (S), and the word list (W) exceeds the context window of Llama 3 models. We use the medium-length chunk $G^m$ and the parallel sentence list $S$ as input for our ICL baseline.

\subsection{\qasper}
\qasper is a benchmark for evaluating the ability of large language models to answer questions about scientific papers~\cite{dasigi2021dataset}. 
To create a challenging multi-query long-context setting resembling the setup described in \Cref{sec:problem-setup}, we concatenate 16 papers all related to \textit{QA NLP models} to form out corpus $\ctx$. 
In total, there are 78 questions about these 16 papers in the dataset, which we use as the queries $\queries$.

Because the dataset only includes short answers and ground-truth spans containing evidence for each answer, we rewrite the answers in a longer, more conversational format using GPT-4.1 and use these as the targets when evaluating.

\section{Theoretical analysis: Relationship between attention, linear attention, and \artifacts}
\label{app:theory}

When we generate text with an autoregressive Transformer, we have to maintain a KV-cache that grows linearly with the length of the input and text. 
In \Cref{app:related-work-reduce-arch}, we discussed a number of architectural modifications that either reduce the size of the KV-cache or do away with it altogether. In particular, when generating text with linear attention (\eg \cite{arora2024simple}), we only need to maintain a constant-sized object -- the KV-state matrix -- during generation. 

Like the KV-state matrix in linear attention, \artifacts consume a constant amount of memory (\ie their size is a hyperparameter, which can be set independently of the input length). 
However, they differ from the KV-state in how they are updated. In this work, \artifacts are updated using \method -- gradient descent on synthetically generated data. On the other hand, KV-states are updated using a linear attention update rule. 

In this section, we will study the update rules for attention, linear attention, and gradient descent when applied to the multi-query associative recall (MQAR) problem~\cite{arora2023zoologymeasuringimprovingrecall}, a popular synthetic benchmark task used for studying the capabilities of long-context architectures. 
In particular, we consider a variant of the standard MQAR problem where key-value pairs are repeated. 
First, we highlight some equivalences between the update rules of these approaches in the case where input keys are orthonormal.
Then, in the more challenging case where input keys are in a Johnson-Lindenstrauss embedding, we provide a separation result showing that the gradient descent update rule is able to exactly solve an MQAR problem that linear attention cannot. 

These theoretical results provide intuition for why constant-sized \artifacts are able to match the performance of full KV-caches in long-context settings when linear-attention architectures have struggled to do so. 

\subsection{Notation}
All vectors are assumed to be row vectors.

Parenthesized superscripts (e.g. $\vk^{(1)}$) denote some temporal quality of an element.  Subscripts denote different elements in a set, as is standard.

A concise explanation for each variable:
\begin{itemize}[leftmargin=*]
    \item $\modelDim:$ model (and token) dimension. 
    \item $\numPairs:$ number of unique key-value pairs.
    \item $\numQueries:$ number of queries.
    \item $\contextSize:$ number of key-value pairs in stream. 
\end{itemize}

\subsection{MQAR}
We define the Multiple Query Associative Recall (MQAR) problem. 
\begin{definition}
    There is a universe of keys: 
\[ K \subset \R^{1\times\modelDim},\]
and values: 
\[ V \subset \R^{1\times\modelDim}.\]
\end{definition}


\begin{definition}
~\cite{arora2023zoologymeasuringimprovingrecall} In the MQAR problem, the input is: 
\[\kvpair{1}, \hdots, \kvpair{\contextSize} \text{ where } \kvpair{\timestep} \in K\times V \text{ for } 1\le \timestep \le \contextSize,\] 

followed by a set of queries \[\vq_{1}, \hdots \vq_{\numQueries} \text{ where } \vq_i \in K \text{ for } 1\le i \le \numQueries.\] 

Then for each $i\in[\numQueries]$, output:
\[
\begin{cases}
    \vv_{i^{*}} \text{ where } i^{*} = \max\{i\in[1,\contextSize] | \vk_{i}=\vq_{j}\} \\
    \bm{0}^{\modelDim} \text{ if no such } i \text{ exists.}
\end{cases}
\]

\end{definition}

\subsection{$\mathbf{\numPairs}-\repetitiveMQAR$}
\begin{definition}
$\numPairs-\repetitiveMQAR$ is a special case where each $(K^{(t)}, V^{(t)})\in \kvSet$, where:
\[S = \{(\vk_1, \vv_1), \hdots, (\vk_{\numPairs}, \vv_{\numPairs})\}.\]
Additionally, $\vk_{i}$ is unique. 
\end{definition}

\begin{definition}
To capture this, $r_{i}^{(\timestep)}$ is defined as the number of occurrences of $(\vk_{i}, \vv_{i})$ in the stream at timestep $\timestep$.
\end{definition}

\subsubsection{Orthonormal Embedding}
First, we will look at the MQAR problem in a restricted case, when all keys are orthonormal. 
\begin{definition}\label{def:orthonorm}
We call the set $K$ to be orthonormal if for all $ \ \vk,\vk'\in K$:
    \[ \langle \vk, \vk' \rangle = 
    \begin{cases} 
        0 & \text{ if } \vk \neq \vk' \\
        1 & \text{ otherwise.}
    \end{cases}
\]
\end{definition}

\subsubsection{Johnson-Lindenstrauss Embedding}
Next, we will look at the MQAR problem in a restricted case, when all keys are in a JL embedding.
\begin{definition}\label{def:JL}
Let $\eps>0$, we call the set $K$ to be $\eps-$JL if for all $\quad \vk,\vk'\in K$:
\[  \langle \vk, \vk' \rangle = 
    \begin{cases} 
        [-\eps, \eps] & \text{ if } \vk \neq \vk' \\
        1 & \text{ otherwise.}
    \end{cases}.
\]
\end{definition}

\subsection{Model Definitions}
Below, we will describe three different model architectures. While they each exhibit different performance and capabilities they can be describe with a common framework for the MQAR problem.
\begin{enumerate} [leftmargin=*]
\item State: is how the model store Key-Value pairs. \item Update rule: how the model incorporates new Key-Value pairs into its state.
\item Query rule: how the model uses its state to answer a look up a value or a query. 
\end{enumerate}
\subsubsection{Transformer}

\begin{enumerate}[leftmargin=*]
    \item The state is: \[\mW^{(\timestep)} = (\mK^{(\timestep)}, \mV^{(\timestep)}),\] where, \[\mK^{(\timestep)}\in\R^{\timestep\times\modelDim},\mV^{(\timestep)}\in\R^{\timestep\times\modelDim}.\] Note that this consumes more memory as the context gets longer. 
    \item The update rule is:
    \[\mK^{(\timestep+1)}=\mK^{(\timestep)}\oplus \vk^{(\timestep+1)}, \mV^{(\timestep+1)}=\mV^{(\timestep)}\oplus \vv^{(\timestep+1)}\]
    \item On query $\vq \in K,$ return: 
    \[\vq \left(\mK^{(\timestep)}\right)^{\top}\mV^{(\timestep)}.\]
    
\end{enumerate}
These rules define the transformer setting for MQAR.

\subsubsection{Linear Attention}
\begin{enumerate}[leftmargin=*]
    \item The state: \[\state^{(\timestep)}\in\R^{\modelDim\times\modelDim}.\]
    \item The update rule is defined as: 
    \[\state^{(\timestep+1)}=\state^{(\timestep)}+(\vk^{(\timestep+1)})^{\top}(\vv^{(\timestep+1)}).\] With the initial matrix being initialized to zeros. I.e. $\state^{(0)} = \bm{0}^{\modelDim\times\modelDim}$. 
    \item On query q, return:
    \[\vq\mW^{(\timestep)}.\]
\end{enumerate}
\begin{lemma}
    ~\cite{yang2025parallelizinglineartransformersdelta} Linear attention rule emerges if we were to update using the loss function $-\vk^{(t)}\mW^{(t)}\vv^{t}$.
\end{lemma}

It is important to mention here that we are not using any kernels for linear attention. 
These rules define the linear attention setting for MQAR.
\begin{lemma}\label{Lemma: DeltaRule}
    ~\cite{yang2025parallelizinglineartransformersdelta} $\cacheMatrixTime{\timestep+1}=\cacheMatrixTime{\timestep}-\keyT{\timestep}\key{\timestep}\cacheMatrixTime{\timestep}+\keyT{\timestep}\val{\timestep}$ is the update rule that emerges when we use the gradient descent loss function: $\frac{1}{2}||\bm{k}^{(\timestep)}\mW^{(\timestep)}-\vv^{(\timestep)}||_{2}^{2}$. 
\end{lemma}
\begin{definition}
    \[\mathcal{L} = \frac{1}{2}||\bm{k}^{(\timestep)}\mW^{(\timestep)}-\vv^{(\timestep)}||_{2}^{2} \]
\end{definition}
    \begin{proof}
        In general, gradient descent has the update rule:
        \begin{equation}\label{eq:update}
           \mW^{(\timestep+1)}=\mW^{(\timestep)}-\eta\nabla_{\mW^{(\timestep)}}. 
        \end{equation}
        Taking the gradient of the loss function gives us:
        \begin{align*}
             \nabla_{\mW} \frac{1}{2}||\bm{k}^{(\timestep)}\mW^{(\timestep)}-\vv^{(\timestep)}||_{2}^{2} &= \left(\vk^{(\timestep)}\right)^{\top} (\vk^{(\timestep)}\mW^{(\timestep)}-\vv^{(\timestep)})\\
            &=\keyT{\timestep}\key{\timestep}\cacheMatrixTime{\timestep}-\keyT{\timestep}\val{\timestep}.
        \end{align*}
    
        Using the above and choosing $\eta=1$, we get for \Cref{eq:update}
        \begin{align*}
            \cacheMatrixTime{\timestep+1} &= \cacheMatrixTime{\timestep}-1\left(\keyT{\timestep}\key{\timestep}\cacheMatrixTime{\timestep}-\keyT{\timestep}\val{\timestep}\right)\\
            &=\cacheMatrixTime{\timestep}-\keyT{\timestep}\key{\timestep}\cacheMatrixTime{\timestep}+\keyT{\timestep}\val{\timestep} .\\ 
        \end{align*}
   \end{proof} 

\subsubsection{Gradient Descent}
Gradient descent training on the cache. We look at the capability of this trained state on a certain input.

\begin{enumerate}[leftmargin=*]
    \item The state at time $\timestep$ is defined as: \[\cacheMatrixTime{\timestep}\in\R^{\modelDim\times\modelDim}.\]
    \item The update rule which follows from \Cref{Lemma: DeltaRule}: 
    \[\cacheMatrixTime{\timestep+1}=\cacheMatrixTime{\timestep}-\keyT{\timestep}\key{\timestep}\cacheMatrixTime{\timestep}+\keyT{\timestep}\val{\timestep}.\] With the initial matrix being initialized to zeros. I.e. $\cacheMatrixTime{0} = \bm{0}^{\modelDim\times\modelDim}$.
    \item On query q, return:
    \[\vq\mW^{(\timestep)}.\]
\end{enumerate}

\subsubsection{Orthonormal Case} 
We now see how the three models perform on the $\numPairs-\repetitiveMQAR$ when $K$ is orthonormal.

\textbf{Transformer}
\begin{lemma}
On every input to MQAR (even those for 1-rep-MQAR) the state of Transformer needs $\Omega(\contextSize\modelDim)$ parameters.
\end{lemma}

Intuitively, at each timestep, you will append $\modelDim$ parameters to the state. At timestep $t$ the model will have $\timestep\modelDim$ parameters.

\textbf{Linear attention}
\begin{theorem}\label{Orthogonal:LA}
Linear attention can solve $\repetitiveMQAR$ for any $\numPairs\ge1$ and orthonormal $K$, up to scaling (producing $r_{i}^{(\timestep)}\vv_{i}$ when $\mW^{(\timestep)}$ is queried with $\vk_i$) and all keys being distinct with $O(\modelDim^2)$ parameters.
\end{theorem}

\begin{proof}
    We first prove that for any $\timestep\ge0$:

    \begin{equation}\label{eq:orthola} \state^{(\timestep)} = \sum_{i'=1}^{\numPairs} r_{i'}^{(t)} \vk_{i'}^{\top} \vv_{i'}.
    \end{equation}

    \textbf{Base Case:} Initially, $\mW^{(0)}=\bm{0}^{\modelDim\times\modelDim}$. From this, we indeed have:
    \[ \state^{(0)} = \sum_{i'=1}^{\numPairs} r_{i'}^{(0)} \vk_{i'}^{\top} \vv_{i'},\]
    
    since for all $i'\in[\numPairs]$: 
    \[r_{i'}^{(0)} = 0.\]

    \textbf{Inductive hypothesis:} Assume that the state matrix at some arbitrary integer timestep $\timestep$ is as claimed. I.e.:  
    \[ \state^{(\timestep)} = \sum_{i'=1}^{\numPairs} r_{i'}^{(t)} \vk_{i'}^{\top} \vv_{i'}.\]

    \textbf{Inductive step:}
    If $\kvpair{j}$ appears at timestep $\timestep+1$ the update rule will be: 
    \begin{align*}\state^{(\timestep+1)} &=\state^{(\timestep)}+(\vk^{(\timestep+1)})^{\top}\vv^{(\timestep)} \\
    &= \state^{(\timestep)}+(\vk_{j})^{\top}\vv_{j}
    \end{align*}

    By the inductive hypothesis, we have that: 
    \begin{align*}
        \state^{(\timestep+1)} 
        &= \state^{(\timestep)}+\vk_{j}(\vv_{j})^{\top} \\
        &= \sum_{i'=1}^{\numPairs} r_{i'}^{(t)} \vk_{i'}^{\top} \vv_{i'}+\vk_{j}(\vv_{j})^{\top} \\
        &= \sum_{i'=1}^{\numPairs} r_{i'}^{(t+1)} \vk_{i'}^{\top} \vv_{i'}.
    \end{align*}

    The final step follows from the fact that $r_{j}^{(\timestep+1)}= r_{j}^{(\timestep)}+1$  when $\kvpair{\timestep+1}=(\vk_{j}, \vv_{j})$ and $r_{i}^{(\timestep+1)}= r_{i}^{(\timestep)}$ for all $i\neq j$.\\
    The proof of \Cref{eq:orthola} is complete by induction. 
    
    Finally, it is the case that on query $\vk_{i}$:
    \begin{align*}
        \vk_{i}\mW^{(\timestep)} &= \vk_{i}\sum_{i'=1}^{\numPairs} r_{i'}^{(t)} \vk_{i'}^{\top} \vv_{i'} \\
        &= \sum_{i'=1}^{\numPairs} r_{i'}^{(t)} \vk_{i}\vk_{i'}^{\top} \vv_{i'} \\
        &= \sum_{i'\neq i} r_{i'}^{(t)} \vk_{i}\vk_{i'}^{\top} \vv_{i'} +  r_{i}^{(t)} \vk_{i}\vk_{i}^{\top} \vv_{i}\\
        &= \sum_{i'\neq i} r_{i'}^{(t)}\cdot 0\cdot \vv_{i'} + r_{i}^{(t)}  \cdot 1 \cdot \vv_{i}\\
        &=   r_{i}^{(t)}\cdot\vv_{i},\\
    \end{align*}
    as desired. In the above, the second last inequality follows from from \Cref{def:orthonorm} and the fact that all $\vk_{i}$ are distinct. 

    $O(\modelDim^{2})$ parameters are needed as the matrix must have dimension $\modelDim\times\modelDim$
\end{proof}

\textbf{Gradient Descent}
\begin{theorem}
    Gradient descent is able to exactly solve the $\numPairs-\repetitiveMQAR$ (produce $\vv_{i}$ when $\mW^{(\timestep)}$ is queries with $\vk_{i}$) with  $O(\modelDim^2)$ parameters.
\end{theorem}

\begin{proof}
    Here we can handle repetitions because our update rule includes a "peel" term. This means it removes the current value stored under a key before updating it with a new value. 

    We will show by induction that for all $\timestep\ge 0$:
    \[
        \cacheMatrixTime{t} = \sum_{i'=1}^{\numPairs} \mathbb{1}_{r^{(t)}_{i'}>0} \cdot \vk_{i'}^{\top}\vv_{i'}.
    \]

    \textbf{Base Case:} Initially, the cache matrix is set to all zeros. From this, naturally follows that:
    \[
        \cacheMatrixTime{0} = \sum_{i'=1}^{\numPairs} 0 \cdot \vk_{i'}^{\top}\vv_{i'},
    \]
    since for all $i'$ 
    
    \[ r^{(0)}_{i'} = 0.\]

    \textbf{Inductive hypothesis:} Assume that at some arbitrary timestep $\timestep$, we have:
    \[
        \cacheMatrixTime{t} = \sum_{i'}^{\numPairs} \mathbb{1}_{r^{(t)}_{i'>0}} \cdot \vk_{i'}^{\top}\vv_{i'}
    \]

    \textbf{Inductive step:}
    If $(\vk_{\ell}, \vv_{\ell})$ appears at timestep $\timestep+1$ the update will be: 
    \begin{align*}
        \sum_{i=1}^{\numPairs} \mathbb{1}_{r^{(t+1)}_{i>0}} \vk_{i}^{\top} \vv_{i} &= \left(\sum_{i'=1}^{\numPairs}\mathbb{1}_{r^{(t)}_{i'>0}} \vk_{i'}^{\top} \vv_{i'} \right) - \left( \sum_{i'=1}^{\numPairs}  \mathbb{1}_{r^{(t)}_{i'>0}} \vk_{\ell}^{\top}\vk_{\ell} \vk_{i'}^{\top} \vv_{i'} \right) + \vk_{\ell}^{\top} \vv_{\ell} \\
        &\text{the second term reduces to just peeling the term relating to $\vk_{\ell}$, if it exists, as all other inner products are $0$,} \\
        &= \left( \sum_{i'=1}^{\numPairs}\mathbb{1}_{r^{(t)}_{i'>0}} \vk_{i'}^{\top} \vv_{i'} \right) - \left(   \mathbb{1}_{r^{(t)}_{\ell>0}} \cdot  \vk_{\ell}^{\top} \vv_{\ell} \right) + \vk_{\ell}^{\top} \vv_{\ell} \\
        &= \left(\sum_{i'\neq \ell}^{\numPairs}\mathbb{1}_{r^{(t)}_{i'>0}} \vk_{i'}^{\top} \vv_{i'} \right) + \vk_{\ell}^{\top} \vv_{\ell} \\
    \end{align*}

    This replaces the value associated with $\vk_{\ell}$ with the new value, while keeping everything else the same. This is the form that we want, as the only time we want to add a key if it is an new key. 

    Finally, it is the case that on query $\vk_{i}$:
    \begin{align*}
        \vk_{i}\cdot \mW^{(\timestep)} &= \vk_{i} \cdot \left(\sum_{i'=1}^{\numPairs} \mathbb{1}_{r^{(t)}_{i'>0}} \vk_{i'}^{\top} \vv_{i'}\right) \\
        &=  \left(\sum_{i'=1}^{\numPairs} \mathbb{1}_{r^{(t)}_{i'>0}} \vk_{i} \cdot\vk_{i'}^{\top} \vv_{i'}\right) \\
        &=  \mathbb{1}_{r^{(t)}_{i>0}} \cdot 1 \cdot \vv_{i}\ \\
        &=  \mathbb{1}_{r^{(t)}_{i>0}}\cdot \vv_{i}\ \\
    \end{align*}
    Again here a matrix of dimension $\modelDim\times\modelDim$ can store $\modelDim$ orthogonal vectors. Thus this requires, $O(\modelDim^{2})$ parameters.
   
\end{proof}

\subsubsection{JL Embedding}
We now see how the 3 models perform on the $\numPairs-\repetitiveMQAR$ when $K$ is $\eps-$JL.

\textbf{Transformer}
\begin{lemma}
On every input to MQAR (even those for 1-rep-MQAR) the state of Transformer needs $\Omega(\contextSize\modelDim)$ parameters.
\end{lemma}
We note that when $K$ is $\eps-$JL it is no longer possible to get the exact answer from query rule $\vk_{i}\mW^{(\timestep)}$. Thus, we need to add a decoding step.

\begin{definition}\label{def:Decode}
The output decoding step is $\vv_{i^{*}}$ where:
\[i^{*} = \arg\max_{i'\in[\numPairs]}\langle\vv_{i'}, \vk_{i}\mW^{(\timestep)}\rangle.\]
\end{definition}

\begin{definition} For all $i,j\in [\numPairs]$, define:
\[ \error{i}{j} = \langle \vk_{i}, \vk_{j} \rangle. \]
\end{definition}

\textbf{Linear Attention}
\begin{theorem}
Linear attention (+ decoding as in \Cref{def:Decode}) is unable to solve even the $2-\repetitiveMQAR$ and each $\vv_{i}$ being 1-hot encoding unless $K$ is $\omega\left(\frac{1}{\contextSize}\right)-$JL. 
\end{theorem}
\begin{proof}
    Due to the agreeance between different keys, when querying for key $i$, there is noise from other keys returned along with the correct answer. While we can tolerate some error, this error scales with the number of times the model has seen a single key. Making it unfit for longer contexts, or contexts with many repeats. 

    First, note that the base case \Cref{eq:orthola} from \Cref{Orthogonal:LA} still holds. In general, this holds for all $K$.
    
    Specifically, on query $\vk_{1}$ we have:
    
    \begin{align*}
        \vk_{1}\mW^{(\timestep)}=r_{1}^{(\timestep)}\langle\vk_{1},\vk_{1}\rangle\vv_{1} + r_{2}^{(\timestep)}\langle\vk_{1},\vk_{2}\rangle\vv_{2}
        &= r_{1}^{(\timestep)}\vv_{1}+r_{2}^{(\timestep)}\error{1}{2}\vv_{2}.
    \end{align*}

    Now, consider an input to $2-\repetitiveMQAR$ such that 
    \[r_{1}^{(\timestep)}<r_{2}^{(\timestep)}\error{1}{2}.\] Note that in this case:
    \[r_{1}^{(\timestep)}= \langle \vv_{1}, \vk_{1}\mW^{(\timestep)}\rangle < \langle \vv_{2}, \vk_{1}\mW^{(\timestep)}\rangle = r_{2}^{(\timestep)}\eps_{1,2}\] 
    and hence we output $\vv_{2}$ instead of $\vv_{1}$.

    If the embedding was $\omega(\frac{1}{N}$ the number of repeats could not overcome the $\eps$ value. 
    
\end{proof}

\textbf{Gradient Descent}
\begin{theorem}
       Gradient descent (+ decoding as in \Cref{def:Decode}) is able to exactly solve $\numPairs-\repetitiveMQAR$ with $O(\modelDim^2)$ parameters for $\eps-$JL $K$, as long as $\eps \le \frac{1}{\numPairs^{2}(\numPairs-1)}$ and $\alpha<\frac{\numPairs-1}{\numPairs+1}$.
\end{theorem}

\begin{proof}

    We define: 
    \[ \errs{i}{j}{\timestep} \] 
    to be the coefficient associated with $\vk_i^{\top} \vv_{j}$ in $\mW^{(\timestep)}$. Specifically, let
\begin{equation}\label{Def: CacheMatrix}
        \cacheMatrixTime{\timestep} = \sum_{i=1}^{\numPairs} \sum_{j=1}^{\numPairs} \errs{i}{j}{\timestep} \vk_{i}^{\top} \vv_{j}
\end{equation}

    We will prove by induction that:
    \begin{equation}\label{eq:errs}\errs{i}{j}{t}=\mathbb{1}_{(\vk_{i},\vv_{j}) \text{ has occurred}} + \Delta^{(\timestep)}_{i,j}
    \end{equation}
    where, 
    \begin{equation}\label{eq:delta}\left|\Delta_{i,j}^{(\timestep)}\right| \le \sum_{a=1}^{\timestep}((\numPairs-1)\epsilon)^{a}.
    \end{equation}
    \textbf{Base Case:} Initially, the state is set to all zeros. From this, naturally follows that  all of the $\errs{i}{j}{\timestep}$ are zero. I.e. \Cref{eq:errs}:
    \[\Delta_{i,j}=0.\]
        
    \textbf{Inductive hypothesis:} Assume that all for some timestep $\timestep$ and $1\le i, j \le \numPairs$: 
    \[\errs{i}{j}{t}=\mathbb{1}_{(\vk_{i},\vv_{j}) \text{ has occurred}} + \Delta^{(\timestep)}_{i,j},\]
    where $\Delta_{i,j}^{(\timestep)}$ satisfies \Cref{eq:delta}.

    \textbf{Inductive Step:}
    If at timestep $\timestep+1$ we are given $(\vk_{\ell}, \vv_{\ell})$, from \Cref{Def: CacheMatrix} the update looks like:
    \begin{align*}
    \mW^{(\timestep+1)}&=\sum_{i=1}^{\numPairs} \sum_{j=1}^{\numPairs} \errs{i}{j}{\timestep+1} \vk_{i}^{\top} \vv_{j}\\ &= \sum_{i'=1}^{\numPairs} \sum_{j'=1}^{\numPairs} \errs{i'}{j'}{\timestep} \vk_{i'}^{\top} \vv_{j'} -
    \left(\sum_{i'=1}^{\numPairs} \sum_{j'=1}^{\numPairs} \errs{i'}{j'}{\timestep}\vk_{\ell}^{\top}\vk_{\ell} \vk_{i'}^{\top} \vv_{j'} \right) + \vk_{\ell}^{\top}\vv_{\ell} \\
    &= \sum_{i'=1}^{\numPairs} \sum_{j'=1}^{\numPairs} \errs{i'}{j'}{\timestep} \vk_{i'}^{\top} \vv_{j'} -
    \left(\sum_{i'=1}^{\numPairs} \sum_{j'=1}^{\numPairs} \error{\ell}{i'} \errs{i'}{j'}{\timestep} \vk_{\ell}^{\top} \vv_{j'} \right) + \vk_{\ell}^{\top}\vv_{\ell} \\ 
     &\text{change the associativity of the summations,} \\ 
    &= \sum_{i'=1}^{\numPairs} \sum_{j'=1}^{\numPairs} \errs{i'}{j'}{\timestep} \vk_{i'}^{\top} \vv_{j'} -
    \left(\sum_{j'=1}^{\numPairs} \left(\sum_{i'=1}^{\numPairs} \error{\ell}{i'} \errs{i'}{j'}{\timestep}\right) \vk_{\ell}^{\top} \vv_{j'} \right) + \vk_{\ell}^{\top}\vv_{\ell} \\ 
    &\text{here we separate the first term where $i'=\ell$ and $i'\neq\ell$,}\\
    &= \sum_{i'\neq\ell}^{\numPairs} \sum_{j'=1}^{\numPairs} \errs{i'}{j'}{\timestep} \vk_{i'}^{\top} \vv_{j'} +  \sum_{j'=1}^{\numPairs} \errs{\ell}{j'}{\timestep} \vk_{\ell}^{\top} \vv_{j'} -
    \left(\sum_{j'=1}^{\numPairs} \left(\sum_{i'=1}^{\numPairs} \error{\ell}{i'} \errs{i'}{j'}{\timestep}\right) \vk_{\ell}^{\top} \vv_{j'} \right) + \vk_{\ell}^{\top}\vv_{\ell} \\ 
    &\text{here we separate the first term where $i'=\ell$ and $i'\neq\ell$,}\\
    &= \sum_{i'\neq\ell}^{\numPairs} \sum_{j'=1}^{\numPairs} \errs{i'}{j'}{\timestep} \vk_{i'}^{\top} \vv_{j'} +  \sum_{j'=1}^{\numPairs} \errs{\ell}{j'}{\timestep} \vk_{\ell}^{\top} \vv_{j'} - 
    \left(\sum_{j'=1}^{\numPairs}  \error{\ell}{\ell} \errs{\ell}{j'}{\timestep}\vk_{\ell}^{\top} \vv_{j'} \right)-\left(\sum_{j'=1}^{\numPairs} \left(\sum_{i'\neq\ell} \error{\ell}{i'} \errs{i'}{j'}{\timestep}\right) \vk_{\ell}^{\top} \vv_{j'} \right) + \vk_{\ell}^{\top}\vv_{\ell} \\ 
    &\text{remove $\eps_{j,j}$,}\\
    &= \sum_{i'\neq\ell}^{\numPairs} \sum_{j'=1}^{\numPairs} \errs{i'}{j'}{\timestep} \vk_{i'}^{\top} \vv_{j'} +  \sum_{j'=1}^{\numPairs} \errs{\ell}{j'}{\timestep} \vk_{\ell}^{\top} \vv_{j'} - 
    \sum_{j'=1}^{\numPairs} \errs{\ell}{j'}{\timestep}\vk_{\ell}^{\top} \vv_{j'} -\left(\sum_{j'=1}^{\numPairs} \left(\sum_{i'\neq\ell} \error{\ell}{i'} \errs{i'}{j'}{\timestep}\right) \vk_{\ell}^{\top} \vv_{j'} \right) + \vk_{\ell}^{\top}\vv_{\ell} \\ 
    &\text{cancel terms,}\\
    &= \sum_{i'\neq\ell}^{\numPairs} \sum_{j'=1}^{\numPairs} \errs{i'}{j'}{\timestep} \vk_{i'}^{\top} \vv_{j'} -\left(\sum_{j'=1}^{\numPairs} \left(\sum_{i'\neq\ell} \error{\ell}{i'} \errs{i'}{j'}{\timestep}\right) \vk_{\ell}^{\top} \vv_{j'} \right) + \vk_{\ell}^{\top}\vv_{\ell} .\\ 
    \end{align*}

    Note with this we can see that: 
    \begin{align*} 
    \errs{i}{j}{\timestep+1} 
    &= \begin{cases}
         \errs{i}{j}{\timestep} &\text{ if } \ell \neq i \\
         -\displaystyle\sum_{i'\neq \ell} \error{\ell}{i'}\errs{i'}{j}{\timestep} + \mathbb{1}_{j=\ell}  &\text{ if } \ell = i 
    \end{cases}.
    \end{align*}
Thus, if $i\neq \ell$, we have:

\[\errs{i}{j}{t+1}=\errs{i}{j}{t},\]
for $i\neq\ell$.
The inductive statement holds for these pairs. 
Now let's consider $\errs{\ell}{j}{t+1}$. If $\ell = j $ then:
\[\errs{\ell}{\ell}{t+1}= 1+ \Delta_{\ell,\ell}^{(\timestep+1)}= \sum_{i'\neq \ell} \error{\ell}{i'}\errs{i'}{j}{\timestep}+1\]

and note that by the triangle inequality and \Cref{def:JL}:
\begin{align*}\left|\Delta_{\ell,\ell}^{(\timestep+1)}\right| &\le \epsilon \sum_{i'\neq \ell}\left|\errs{i'}{\ell}{\timestep}\right| \\ 
&\text{by the inductive hypothesis,} \\ 
&\le \epsilon\sum_{i'\neq \ell}(1+\sum_{a=1}^{\timestep}((\numPairs-1)\epsilon)^{a})\\
&= ((\numPairs-1)\epsilon)(1+\sum_{a=1}^{\timestep}((\numPairs-1)\epsilon)^{a}) \\
&= (\sum_{a=1}^{\timestep+1}((\numPairs-1)\epsilon)^{a}),  
\end{align*}
as desired.

Then for $j\neq \ell$, we have:
\begin{align*} \left|\Delta^{(\timestep+1)}_{j,\ell}\right| &=\left|\errs{i}{j}{t+1}\right|\\ &= \left|\sum_{i'\neq\ell} \error{\ell}{i'}\errs{i'}{j}{\timestep}\right|
\end{align*}
The bounding of $\Delta_{\ell, j}^{(\timestep)}$ is similar to the $\ell=j$ case. 

With this we have completed the inductive proof on error terms. 

If the we set: 

\[\epsilon < \frac{1}{m^{2}(\numPairs-1)}, \]

we get the following bound:

\begin{align} \Delta_{i,j}^{(\timestep)}&\le \sum_{a=1}^{\timestep}((\numPairs-1)\epsilon)^{a}  \\ &\le \frac{(\numPairs-1)\epsilon}{1-(\numPairs-1)\epsilon} \\
&<\frac{1}{\numPairs^{2}-1} \label{eq:delta-bound}
\end{align}


Before the next steps, we must bound:
\begin{equation} \label{eq:value bound}
    \left|\langle \vv_{i}, \vv_{j} \rangle \right| \le \alpha
\end{equation}

For a query with $\vk_{i}$, assuming we have seen $\vk_{i}$ before, we get:
\begin{align*}
    \vk_{i}\cdot\mW^{(\timestep)} &=\vv_{i} + \sum_{j'\neq i} \Delta_{i,j'}^{(\timestep)} \vv_{j'}
\end{align*}

Now for the decoding step where for an arbitrary $\vv_{j}$ we get:
\begin{align*}
    \langle \vv_{j}, \vk_{i}\cdot\mW^{(\timestep)} \rangle &=\langle \vv_{j}, \vv_{i}\rangle + \langle \vv_{j}, \sum_{j'\neq i} \Delta_{i,j'} \vv_{j'} \rangle
\end{align*}

For the case where $i=j$ it is the case that:
\begin{align*}\langle \vv_{i}, \vk_{i}\cdot\mW^{(\timestep)} \rangle &=1 + \langle \vv_{i}, \sum_{j'\neq i} \Delta_{i,j'} \vv_{j'} \rangle \\ 
&\ge 1 - \frac{1}{m+1} \alpha.
\end{align*}
This follows from \Cref{eq:delta-bound} and \Cref{eq:value bound}.

For the case where $i\neq j$ it is the case that:
\begin{align*}\langle \vv_{j}, \vk_{i}\cdot\mW^{(\timestep)} \rangle &= \langle \vv_i, \vv_j \rangle + \langle \vv_{j}, \sum_{j'\neq i} \Delta_{i,j'} \vv_{j'} \rangle \\ 
&\le \alpha + \frac{1}{m+1} \alpha
\end{align*}
This follows from \Cref{eq:delta-bound} and \Cref{eq:value bound}.

As a result, we will always pick the correct value when $\alpha < \frac{m-1}{m+1}$. 
\end{proof}



\end{document}